\documentclass[11pt]{article}
 
\usepackage{mypreamble}

\usepackage{enumitem}

\usepackage{url}
 
\begin{document}

\title{A Bi-level Nonlinear Eigenvector Algorithm for\\
Wasserstein Discriminant Analysis}
\author{Dong Min Roh\thanks{Department of Mathematics, University of California, Davis, CA 95616, USA, \href{droh@ucdavis.edu}{droh@ucdavis.edu}} 
\and Zhaojun Bai\thanks{Department of Computer Science, University of California, Davis, CA 95616, USA, \href{zbai@ucdavis.edu}{zbai@ucdavis.edu}}
\and Ren-Cang Li\thanks{Department of Mathematics, University of Texas at Arlington, Arlington, TX 76019, USA, \href{rcli@uta.edu}{rcli@uta.edu}}
}

 
\date{}
 
\maketitle

\interfootnotelinepenalty=100000000


\begin{abstract}
Much like the classical Fisher linear discriminant analysis (LDA), the recently proposed Wasserstein discriminant analysis (WDA) is a linear dimensionality reduction method that seeks a projection matrix to maximize the dispersion of different data classes and minimize the dispersion of same data classes via a bi-level optimization. In contrast to LDA, WDA can account for both global and local interconnections between data classes by using the underlying principles of optimal transport. 
In this paper, a bi-level nonlinear eigenvector algorithm (WDA-nepv) is presented to fully exploit the structures of the bi-level optimization of WDA.
The inner level of WDA-nepv for computing the optimal transport matrices is formulated as an eigenvector-dependent nonlinear eigenvalue problem (NEPv), and meanwhile, the outer level for trace ratio optimizations is formulated as another NEPv.
Both NEPvs can be computed efficiently under the self-consistent field (SCF) framework.
WDA-nepv is derivative-free and surrogate-model-free when compared with existing algorithms. 
Convergence analysis of the proposed WDA-nepv justifies the utilization of the SCF for solving the bi-level optimization of WDA.
Numerical experiments with synthetic and real-life datasets demonstrate the classification accuracy and scalability of WDA-nepv.
\end{abstract}


\section{Introduction}
As widely used feature extraction approaches in machine learning, dimensionality reduction (DR) methods \cite{van2009dimensionality,burges2010dimension,fodor2002survey,cunningham2015linear} learn projections such that the projected lower dimensional subspaces maintain the coherent structure of datasets and reduce computational costs of classification or clustering. The linear projection obtained from linear DR methods takes the form of a matrix such that the embedding to the lower dimensional subspace only involves matrix multiplications. Due to such ease in interpretation and implementation, linear DR methods are often the favored choice among numerous DR methods. For example, principal component analysis (PCA) \cite{fukunaga2013introduction} seeks to find a linear projection that preserves the dataset's variation and is one of the most common and well-known DR methods. Other well-known DR methods include Fisher linear discriminant analysis (LDA) \cite{fukunaga2013introduction} to take into account the information of classes and compute a linear projection that best separates different classes, and Mahalanobis metric learning \cite{kulis2013metric} to seek a distance metric that better models the relationship among dataset from a linear projection.

Wasserstein discriminant analysis (WDA) \cite{flamary2018wasserstein} is a supervised linear DR that is based on the use of regularized Wasserstein distances \cite{cuturi2013sinkhorn} as a distance metric. 
Similar to Fisher linear discriminant analysis (LDA), WDA seeks  a projection matrix to maximize the dispersion of projected points between different classes and minimize the dispersion of projected points within same classes. An important distinction between LDA and WDA is that while LDA only considers the global relations between data points \cite{cai2007locality,fan2011local,nie2007neighborhood,weinberger2009distance}, WDA can dynamically take into account both global and local information through the choice of a regularization parameter.
A recent study has shown that WDA outperforms other linear DR methods in various learning tasks, such as sequence pattern analysis \cite{su2017order,su2019order}, graph classification \cite{zhang2021deep}, and multi-view classification \cite{kasai2020multi}.

WDA is formulated as a bi-level nonlinear trace ratio optimization \cite{flamary2018wasserstein}. The inner optimization computes the optimal transport (OT) matrices of the regularized Wasserstein distance, an essential factor responsible for WDA's ability to account for both global and local relations. The outer optimization involves a nonlinear trace ratio for the dispersion of data classes.

\paragraph{Related work.} 
The algorithmic aspects of WDA are the focus of this paper.
To the best of our knowledge, there are two algorithms specifically designed for WDA. 
One approach is to use a projected gradient descent method \cite{flamary2018wasserstein}. 
This approach requires the computation of derivatives of a highly nonlinear objective function at a cost that grows quadratically in the dimension size and the number of data points. 
Another work considers a surrogate ratio trace model of WDA \cite{liu2020ratio}. This approach lacks the quantification of the approximation error of the surrogate model. It is well-known that even in the case of the conventional trace ratio optimization with constant matrices, the surrogate ratio trace model could be a poor approximation \cite{wang2007trace}.

\paragraph{Contributions.}
To overcome the theoretical and practical shortcomings of existing algorithms, we propose a new algorithm called WDA-nepv. 
WDA-nepv is a derivative-free and surrogate-model-free algorithm that takes full advantage of the bi-level structure of WDA. Specifically, we formulate the computation of the OT matrices in the inner optimization and the outer trace ratio optimization as two eigenvector-dependent nonlinear eigenvalue problems (NEPvs).

We propose to solve both NEPvs by the self-consistent field (SCF) iteration. 
Fast eigensolvers such as implicitly restarted Arnoldi method \cite{sorensen1997implicitly} enables efficient implementation of SCF. 
In addition, we propose an efficient way to compute the cross-covariance matrices of WDA. Specifically, we convert multiple calls to level-2 BLAS (Basic Linear Algebra Subprograms) to level-3 BLAS, which leads to a significant reduction in computing time. 
We provide a convergence analysis for the SCF framework of WDA-nepv. 
Numerically, we demonstrate linear convergence of WDA-nepv. Moreover, using real-life datasets, we demonstrate that the classification accuracy of WDA-nepv is comparable to or outperforms the other existing algorithms. We also present scalability of WDA-nepv algorithm, demonstrating its computational efficiency. 

We emphasize that the primary purpose of this work is not to demonstrate the advantages of WDA compared with other DR methods. For an in-depth discussion of the relation of WDA with other discriminant analysis and the comparison of its classification accuracy with other DR methods, the reader is referred to \cite{flamary2018wasserstein,liu2020ratio,kasai2020multi,su2019order}.

\paragraph{Paper outline.}
In Section 2, Wasserstein distance and regularized Wasserstein distance are discussed. The setup of WDA and its formulation as a nonlinear trace ratio optimization is presented in Section 3. We give a summary of existing algorithms for WDA and point out their limitations and shortcomings in Section 4. In Section 5, we propose a bi-level algorithm that involves solving NEPv for OT matrices and trace ratio optimizations and provide convergence analysis for the algorithm. Numerical experiments on convergence, classification, and scalability of our algorithm are provided in Section 6. Finally, we conclude our paper in Section 7 and discuss our future plans of study.

\paragraph{Notation.}  
$\RR^{n\times m}$ and $\RR^n$ denote the sets of all $n\times m$ real matrices and all real vectors of size $n$, respectively. 
$\RR_+^n$ and $\RR_+^{n\times m}$ represent the sets of vectors and matrices, respectively, whose components are non-negative.
$\langle\bbf{A},\bbf{B}\rangle:=\tr(\bbf{A}^T\bbf{B}) 
= \sum_{ij} A_{ij}B_{ij}$ represents the 
Frobenius inner product between matrices $\bbf{A},\bbf{B}\in\RR^{n\times m}$.
$\bbf{1}_d\in\RR^d$ denotes the vector of all ones and $I_p\in\RR^{p\times p}$ denotes the identity matrix. 
$\mathcal{D}(\bbf{a})\in\RR^{d\times d}$ denotes a diagonal matrix whose diagonal 
is the vector $\bbf{a}\in\RR^d$.
The symbol $./$ denotes element-wise division between vectors of the same dimension.

\paragraph{Acronyms.}
Accelerated Sinkhorn-Knopp (Acc-SK); Basic Linear Algebra Subprograms (BLAS); dimensionality reduction (DR); linear discriminant analysis (LDA); eigenvector-dependent nonlinear eigenvalue problem (NEPv); optimal transport (OT); principal component analysis (PCA); self-consistent field (SCF); Sinkhorn-Knopp (SK); trace ratio optimization (TRopt); Wasserstein discriminant analysis (WDA);


\section{Wasserstein Distances}\label{sec:dist}
We first present the background on Wasserstein distance \cite{villani2009optimal, cuturi2013sinkhorn}, also known as Earth mover's distance \cite{rubner1997earth}, and its regularized variant known as regularized Wasserstein distance \cite{flamary2018wasserstein}. By principles from optimal transport theory, these distances define a geometry over the space of probability enabling the ability to compare distances between probability distributions. Following the introduction of these distances, we discuss the cost of their computations and conclude the section with a discussion on the applications of these distances.

\paragraph{Transport polytope.}
Given two probability vectors $\bbf{r}\in\RR_+^n$ 
and $\bbf{c}\in\RR_+^m$, i.e., $\bbf{r}^T\bbf{1}_n=1$ and $\bbf{c}^T\bbf{1}_m=1$, the following set of positive matrices
\begin{equation}\label{eq:Upoly}
U(\bbf{r},\bbf{c}):=\bigg\{\bbf{T} \mid \bbf{T}\in\RR_+^{n\times m}, \, \bbf{T}\bbf{1}_m=\bbf{r},\, \bbf{T}^T\bbf{1}_n=\bbf{c}\bigg\}
\end{equation}
is called the {\em transport polytope} of $\bbf{r}$ and $\bbf{c}$. 
In analogy, if we consider each component of $\bbf{r}$ as a pile of sand and each component of $\bbf{c}$ as a hole to be filled in, the component $T_{ij}$ of $\bbf{T}$ represents the amount of sand from the pile $r_i$ that gets moved to the hole $c_j$.
For this reason, a matrix $\bbf{T}\in U(\bbf{r},\bbf{c})$ is referred to as a {\em transport matrix} of $\bbf{r}$ and $\bbf{c}$.

\paragraph{Wasserstein distance.}
Let $\bbf{M}\in\RR^{n\times m}$ be a non-negative matrix whose elements $M_{ij}\geq0$
represents the unit cost of transporting $r_i$ to $c_j$. Then, the cost of 
transporting $\bbf{r}$ to $\bbf{c}$ using a transport matrix $\bbf{T}$ and a cost matrix $\bbf{M}$
is quantified as 
\begin{equation}\label{eq:cost_TM}
\langle\bbf{T},\bbf{M}\rangle=\sum_{ij}T_{ij}M_{ij}.
\end{equation}
Given the matrix $\bbf{M}$,
the optimization problem, referred to as the {\em optimal transport problem} \cite{villani2009optimal,cuturi2013sinkhorn},
\begin{equation}\label{eq:Was_dist}
W(\bbf{r},\bbf{c})
:=\min_{\bbf{T}\in U(\bbf{r},\bbf{c})}\langle\bbf{T},\bbf{M}\rangle
\end{equation}
seeks a transport matrix $\bbf{T}$ that minimizes the transport cost between $\bbf{r}$ and $\bbf{c}$.
When $\bbf{M}$ is a metric matrix, i.e., $\bbf{M}$ belongs in the cone
$$\mathcal{M}=\bigg\{\bbf{M} \mid \bbf{M}\in\RR_+^{n\times m},\, M_{ij}=0 \Leftrightarrow i=j,\mbox{ and } M_{ij}\leq M_{ik}+M_{kj},\, \forall i,j,k\bigg\},$$
then $W(\bbf{r},\bbf{c})$ forms a distance
between $\bbf{r}$ and $\bbf{c}$ \cite{villani2009optimal}, and  
is called the \textit{optimal transport distance} or the \textit{Wasserstein distance}\footnote{Strictly speaking, the correct terminology of $W(\bbf{r},\bbf{c})$~\eqref{eq:Was_dist} is the \textit{1st Wasserstein distance with discrete measure} \cite{villani2009optimal,peyre2019computational}. To be consistent with the terminology in Wasserstein Discriminant Analysis \cite{flamary2018wasserstein}, we will refer to it as the Wasserstein distance and its regularized counterpart as the regularized Wasserstein distance.} between $\bbf{r}$ and $\bbf{c}$.


\paragraph{Regularized Wasserstein distance.}
Computing the Wasserstein distance, however, is subject to heavy costs that scale in super-cubic with respect to the size of the probability vector \cite{pele2009fast}. To reduce such heavy costs, Cuturi \cite{cuturi2013sinkhorn} proposed an entropic constraint on the transport matrix that not only lowers the computation cost of the problem but also smooths the search space. Namely, a convex subset of $U(\bbf{r},\bbf{c})$ was introduced:
\begin{equation}\label{eq:Upoly_alpha}
    U_{\alpha}(\bbf{r},\bbf{c}):=\bigg\{\bbf{T} \mid \bbf{T} \in U(\bbf{r},\bbf{c})\,\,\mbox{and}\,\, h(\bbf{T})\geq h(\bbf{r})+h(\bbf{c})-\alpha\bigg\},
\end{equation}
where $\alpha\geq0$ and $h(\bbf{r})$ and $h(\bbf{T})$ are the entropy of a probability vector and the entropy of a transport matrix, respectively, defined as
\begin{equation}\label{eq:entropy}
h(\bbf{r})=-\sum_{i=1}^nr_i\log r_i
\quad \mbox{and} \quad
h(\bbf{T})=-\sum_{ij}T_{ij}\log T_{ij}.
\end{equation}
Then, Cuturi \cite{cuturi2013sinkhorn} defined the following distance between $\bbf{r}$ and $\bbf{c}$
\begin{equation}\label{eq:Sink_dist}
    W_\alpha(\bbf{r},\bbf{c}):=\min_{\bbf{T}\in U_\alpha(\bbf{r},\bbf{c})}\langle\bbf{T},\bbf{M}\rangle 
\end{equation}
as the \textit{Sinkhorn distance}, such naming due to its solvability by the Sinkhorn-Knopp algorithm.

By considering the Lagrange multiplier for the entropy constraint of the Sinkhorn distance, the optimal solution to~\eqref{eq:Sink_dist}, denoted as $\bbf{T}^\lambda$, can be computed as the solution to the problem \begin{equation}\label{eq:Sink_dist_dual}
\bbf{T}^\lambda = \argmin_{\bbf{T}\in U(\bbf{r},\bbf{c})}\, \Big\{ \langle\bbf{T},\bbf{M}\rangle-\frac{1}{\lambda} h(\bbf{T}) \Big\}.
\end{equation} 
Consistent with the terminology in \cite{flamary2018wasserstein}, we refer to $\bbf{T}^\lambda$ as the \textit{optimal transport (OT) matrix} with $\lambda$ as the \textit{regularization parameter} and
\begin{equation}\label{eq:reg_Was_dist}
W_\lambda(\bbf{r},\bbf{c}):=\langle\bbf{T}^\lambda,\bbf{M}\rangle
\end{equation}
as the \textit{regularized Wasserstein distance} (also known as dual-Sinkhorn divergence \cite{cuturi2013sinkhorn}) between the probability vectors $\bbf{r}$ and $\bbf{c}$. By duality theory, to each $\alpha$ in Sinkhorn distance~\eqref{eq:Sink_dist} corresponds a $\lambda\in[0,\infty]$ in regularized Wasserstein distance~\eqref{eq:reg_Was_dist} such that 
$$
W_\alpha(\bbf{r},\bbf{c})=W_\lambda(\bbf{r},\bbf{c}).
$$
Moreover, as the regularized parameter $\lambda\to\infty$, the regularized Wasserstein distance $W_\lambda(\bbf{r},\bbf{c})$ approaches the Wasserstein distance $W(\bbf{r},\bbf{c})$ \cite{flamary2018wasserstein}. 

\paragraph{Structure of the OT matrix $\bbf{T}^\lambda$.}
The entropy smoothed optimal transport problem~\eqref{eq:Sink_dist_dual} looks to minimize the total transport cost while maximizing the entropy of the transport matrix. Both of these terms are convex which makes the problem~\eqref{eq:Sink_dist_dual} a convex problem. Thus, the OT matrix $\bbf{T}^\lambda$ exists and is unique \cite{cuturi2013sinkhorn}. Furthermore, the OT matrix $\bbf{T}^\lambda$ admits a simple structure based on the first order analysis and Sinkhorn's Theorem~\cite{sinkhorn1967diagonal}.
\begin{theorem}\label{thm:Tlambda}
For $\lambda>0$, the solution $\bbf{T}^\lambda$ to 
the entropy smoothed optimal transport problem~\eqref{eq:Sink_dist_dual} is unique and has the form
\begin{equation}\label{eq:T}
    \bbf{T}^\lambda=\mathcal{D}(\bbf{u})\bbf{K}\mathcal{D}(\bbf{v}),
\end{equation}
where $\bbf{u}\in\RR_+^n$, $\bbf{v}\in\RR_+^m$, and 
$\bbf{K}:=e^{-\lambda\bbf{M}}$ is the element-wise exponential of $-\lambda\bbf{M}$. $\bbf{u}$ and $\bbf{v}$ are uniquely 
defined up to a multiplication factor 
\footnote{Given solutions $\bbf{u}$ and $\bbf{v}$ to \eqref{eq:T}, for any positive nonzero scalar $\alpha$, $\alpha\bbf{u}$ and $\frac{1}{\alpha}\bbf{v}$ are also solutions to \eqref{eq:T}.} 
by the relation 
\begin{eqnarray} \label{eq:uvdef} 
\begin{cases} 
\bbf{u} =\bbf{r}./(\bbf{K}\bbf{v})  \\ 
\bbf{v} =\bbf{c}./(\bbf{K}^T\bbf{u}).
\end{cases} 
\end{eqnarray} 
\end{theorem}
\begin{proof} 
See \cite[Lemma 2]{cuturi2013sinkhorn}.
\end{proof} 

Theorem~\ref{thm:Tlambda} indicates that computing the OT matrix $\bbf{T}^\lambda$ is equivalent to the problem of computing the vectors $\bbf{u}\in\RR_+^n$ and $\bbf{v}\in\RR_+^m$ from the given matrix $\bbf{K}\in\RR_+^{n\times m}$ to satisfy \eqref{eq:uvdef}. Such problem belongs to the class of well-known problems called the matrix balancing problems. We leave the discussion of the matrix balancing problem and its algorithms to Section~\ref{subsec:T}.

\paragraph{Costs of Wasserstein distance and regularized Wasserstein distance.} 
Computing the Wasserstein distance~\eqref{eq:Was_dist} comes with a high computational cost.
Assuming $d=n=m$ for simplicity, for any convex optimization algorithms, such as network simplex method or interior point method, the cost of computing the Wasserstein distance~\eqref{eq:Was_dist} scales in super-cubic $O(d^3\log(d))$ \cite{pele2009fast}. On the other hand, as a matrix balancing problem, the regularized Wasserstein distance~\eqref{eq:reg_Was_dist} can be solved efficiently with existing algorithms. For instance, it can be computed with the Sinkhorn-Knopp (SK) iteration \cite{sinkhorn1964relationship,sinkhorn1967diagonal,sinkhorn1967concerning} with a linear convergence rate \cite{franklin1989scaling, knight2008sinkhorn}. The SK algorithm only involves matrix-vector multiplications and scales as $O(d^2)$. We will discuss the SK iteration and its accelerated variants in Section~\ref{subsec:T}.

\paragraph{Applications.}
The regularized Wasserstein distance serves as a powerful tool for computing distances between probability distributions and has found its popularity in machine learning. For example, in sequence pattern analysis, utilizing the regularized Wasserstein distance, a sequence metric that preserves the inherent temporal relationships of the instances in sequences was introduced in \cite{su2017order}. This sequence metric was further used in a discriminant analysis similar to LDA as a supervised linear DR method for sequence data in \cite{su2019order}. In graph classification, the authors in \cite{zhang2021deep} proposed a deep Wasserstein graph discriminant learning method, utilizing graph neural network \cite{scarselli2008graph} and regularized Wasserstein distance as a distance metric.

\paragraph{Remark:} As pointed out in \cite{flamary2018wasserstein}, the entropy smoothed optimal transport problem~\eqref{eq:Sink_dist_dual} of the OT matrix $\bbf{T}^\lambda$ can equivalently be written as
\begin{equation}\label{eq:Sink_dist_dual2}
    \bbf{T}^\lambda=\argmin_{\bbf{T}\in U(\bbf{r},\bbf{c})}\Big\{\lambda\langle\bbf{T},\bbf{M}\rangle- h(\bbf{T})\Big\}.
\end{equation}
We adopt the formulation \eqref{eq:Sink_dist_dual2} for the rest of the paper.


\section{Wasserstein Discriminant Analysis}\label{sec:WDA}
Wasserstein discriminant analysis (WDA) seeks to find a projection matrix that maximizes the dispersion between different classes and minimizes the dispersion within same classes. Meanwhile, WDA can also control global and local inter-relations between data classes using a regularized Wasserstein distance. In this section, we present the formulation of WDA as a nonlinear trace ratio optimization. Then, we conclude the section with discussions on the advantages of WDA.

\paragraph{Data setup.} 
Consider the datasets 
$\{\bbf{x}_i^c\}_{i=1}^{n_c}$ of data points $\bbf{x}_i^c\in\RR^d$ 
for different classes $c=1,2,\ldots,C$, where $n_c$ denotes the number of data points in class $c$.
For each class $c$, a data matrix $\bbf{X}^c\in\RR^{d\times n_c}$ is formed
by constructing its columns to be the vectors $\bbf{x}_i^c$ 
in class $c$:
\[
\bbf{X}^c = [\, \bbf{x}^c_1,\,  \ldots,\,  \bbf{x}^c_{n_c}\, ].
\] 
Without loss of generality, we assume that the data matrices $\bbf{X}^c$ are standardized, i.e., each feature is scaled such that it has mean $0$ and standard deviation $1$. Otherwise, we can always preprocess $\bbf{X}^c$ as 
$$\bbf{X}^c\leftarrow\bbf{X}^c-\frac{1}{n_c}(\bbf{X}^c\bbf{1}_{n_c})\bbf{1}_{n_c}^T$$
to have mean $0$ then divide each feature by its standard deviation to set its standard deviation to be $1$.
Data standardization is a common practice in machine learning algorithms
to transform the features to a common scale while maintaining the structure of the data. This often leads to improved performance in algorithms \cite{han2022data}. 

\paragraph{Definition of Wasserstein discriminant analysis.}
WDA introduced in \cite{flamary2018wasserstein}
presumes empirical measure as the underlying probability measure of the data matrices and uses the regularized Wasserstein distance as the distance metric. Specifically, the regularized Wasserstein distance between two projected data matrices by an orthonormal projection $\bbf{P}\in\RR^{d\times p}$ ($p\ll d)$ is defined as
\begin{align}\label{eq:reg_Was_dist_mat}
W_\lambda(\bbf{P}^T\bbf{X}^c,\bbf{P}^T\bbf{X}^{c'}):=\langle\bbf{T}^{c,c'}(\bbf{P}),\bbf{M}_{\bbf{P}^T\bbf{X}^c,\bbf{P}^T\bbf{X}^{c'}}\rangle,
\end{align}
where 
$\bbf{M}_{\bbf{P}^T\bbf{X}^c,\bbf{P}^T\bbf{X}^{c'}}$ is the cost matrix defined as the Euclidean distances between projected points:
\begin{equation}\label{eq:dist_mat_P}
    \bbf{M}_{\bbf{P^T X}^c,\bbf{P^T X}^{c'}}:=\Big( [\|\bbf{P}^T\bbf{x}_i^c-\bbf{P}^T\bbf{x}_j^{c'}\|_2^2]_{ij} \Big) \in\RR^{n_c\times n_{c'}}
\end{equation}
and $\bbf{T}^{c,c'}(\bbf{P})$ is the OT matrix, defined as the solution of the entropy-smoothed OT problem:
\begin{equation}\label{eq:WDA_otp}
    \bbf{T}^{c,c'}(\bbf{P}):=\argmin_{\bbf{T}\in U_{n_c,n_{c'}}}\, \Big\{\lambda\langle\bbf{T},\bbf{M}_{\bbf{P}^T\bbf{X}^c,\bbf{P}^T\bbf{X}^{c'}}\rangle-h(\bbf{T}) \Big\},
\end{equation}
where $U_{n_c,n_{c'}}$ is the transport polytope between $\bbf{P^T X}^c$ and $\bbf{P^T X}^{c'}$ defined as  
\begin{equation}\label{eq:nonneg_mats}
U_{n_c,n_{c'}}:=\bigg\{\bbf{T} \mid \bbf{T}\in\RR_{+}^{n_c\times n_{c'}},\, 
\bbf{T}\bbf{1}_{n_{c'}}=\frac{1}{n_c}\bbf{1}_{n_c},\, 
\bbf{T}^T\bbf{1}_{n_c}=\frac{1}{n_{c'}}\bbf{1}_{n_{c'}} \bigg\}.
\end{equation}
The regularization parameter $\lambda\geq0$ plays a critical role in dynamically controlling the global and local relations between data points. We highlight the role of $\lambda$ in the discussion of the advantages of WDA.

WDA adopts the formulation of Linear discriminant analysis (LDA) to seek a projection matrix $\bbf{P}\in\RR^{d\times p}$ by solving
\begin{equation}\label{eq:WDA}
\max_{\bbf{P}^T\bbf{P}=I_p}
\frac{\sum_{c,c'>c}W_\lambda(\bbf{P}^T\bbf{X}^c,\bbf{P}^T\bbf{X}^{c'})}
{\sum_{c}W_\lambda(\bbf{P}^T\bbf{X}^c,\bbf{P}^T\bbf{X}^{c})}.
\end{equation}
The numerator and the denominator of~\eqref{eq:WDA} are the sums of the regularized Wasserstein distances~\eqref{eq:reg_Was_dist_mat} between the inter-classes and the intra-classes, respectively. As a maximization problem, WDA~\eqref{eq:WDA} seeks a projection that maximizes the numerator (the dispersion between different classes) and minimizes the denominator (the dispersion within same classes).

\paragraph{Nonlinear trace ratio formulation of WDA.}
By~\eqref{eq:dist_mat_P} and the Frobenius inner product, the inter-class distance between data matrices $\bbf{P}^T\bbf{X}^c$ and $\bbf{P}^T\bbf{X}^{c'}$ is
\begin{align}  W_\lambda(\bbf{P}^T\bbf{X}^c,\bbf{P}^T\bbf{X}^{c'})&=\langle\bbf{T}^{c,c'}(\bbf{P}),\bbf{M}_{\bbf{P}^T\bbf{X}^c,\bbf{P}^T,\bbf{X}^{c'}}\rangle\nonumber\\
    &=\sum_{ij}{T}_{ij}^{c,c'}(\bbf{P})\|\bbf{P}^T\bbf{x}_i^c-\bbf{P}^T\bbf{x}_j^{c'}\|_2^2 \label{eq:weighted_sum} \\
    &=\sum_{ij}{T}_{ij}^{c,c'}(\bbf{P})\tr(\bbf{P}^T(\bbf{x}_i^c-\bbf{x}_j^{c'})(\bbf{x}_i^c-\bbf{x}_j^{c'})^T\bbf{P})\nonumber\\
    &=\tr(\bbf{P}^T[\sum_{ij}{T}_{ij}^{c,c'}(\bbf{P})(\bbf{x}_i^c-\bbf{x}_j^{c'})(\bbf{x}_i^c-\bbf{x}_j^{c'})^T]\bbf{P})\nonumber\\
    &=:\tr(\bbf{P}^T\bbf{C}^{c,c'}(\bbf{P})\bbf{P}),\label{eq:C_def}
\end{align}
where 
\begin{equation}\label{eq:C_cc'}
    \bbf{C}^{c,c'}(\bbf{P}):=
    \sum_{ij}{T}_{ij}^{c,c'}(\bbf{P})(\bbf{x}_i^c-\bbf{x}_j^{c'})(\bbf{x}_i^c-\bbf{x}_j^{c'})^T.
\end{equation}
Similarly, the intra-class distance between $\bbf{P}^T\bbf{X}^c$ to itself is a trace operator
\begin{align}  \label{eq:C_def2} 
W_\lambda(\bbf{P}^T\bbf{X}^c,\bbf{P}^T\bbf{X}^{c}) 
& =:\tr(\bbf{P}^T\bbf{C}^{c,c}(\bbf{P})\bbf{P}),
\end{align} 
where
\begin{equation}\label{eq:C_cc}
    \bbf{C}^{c,c}(\bbf{P}):=
    \sum_{ij}{T}_{ij}^{c,c}(\bbf{P})(\bbf{x}_i^c-\bbf{x}_j^{c})(\bbf{x}_i^c-\bbf{x}_j^{c})^T.
\end{equation}
Thus, by \eqref{eq:C_def} and \eqref{eq:C_def2}, WDA \eqref{eq:WDA} can be reformulated as the nonlinear trace ratio optimization (NTRopt):
\begin{equation}\label{eq:WDA_tropt}
\max_{\bbf{P}^T\bbf{P}=I_p}\left\{
f(\bbf{P}) := \frac{\tr(\bbf{P}^T\bbf{C}_b(\bbf{P})\bbf{P})}{\tr(\bbf{P}^T\bbf{C}_w(\bbf{P})\bbf{P})}\right\},
\end{equation}
where $\bbf{C}_b(\bbf{P})$ and $\bbf{C}_w(\bbf{P})$
are defined as the between and within cross-covariance matrices, respectively:
\begin{align}
    \bbf{C}_b(\bbf{P}) & :=\sum_{c,c'>c}\bbf{C}^{c,c'}(\bbf{P}) 
    = \sum_{c,c'>c}  \sum_{ij}{T}_{ij}^{c,c'}(\bbf{P})(\bbf{x}_i^c-\bbf{x}_j^{c'})(\bbf{x}_i^c-\bbf{x}_j^{c'})^T
    \in \mathbb{R}^{d \times d},  \label{eq:CbCw1} \\
    \bbf{C}_w(\bbf{P}) & :=\sum_{c}\bbf{C}^{c,c}(\bbf{P})
    = \sum_c  \sum_{ij}{T}_{ij}^{c,c}(\bbf{P})(\bbf{x}_i^c-\bbf{x}_j^{c})(\bbf{x}_i^c-\bbf{x}_j^{c})^T 
    \in \mathbb{R}^{d \times d}. \label{eq:CbCw2}
\end{align}
We refer to the WDA problem \eqref{eq:WDA_tropt} as \textbf{NTRopt-WDA}.

\paragraph{Advantages of WDA.}~
Among the number of advantages of WDA \cite{flamary2018wasserstein}, two outstanding ones are that WDA is a generalization of LDA and that WDA can dynamically consider both global and local relations of the data points.
\begin{itemize}
    \item When the regularization parameter $\lambda=0$, each relation between data points is treated equally and NTRopt-WDA~\eqref{eq:WDA_tropt} reduces to LDA. Specifically, by \eqref{eq:Sink_dist_dual2}, the OT matrix is $$\bbf{T}^{c,c'}:=\argmax_{\bbf{T}\in U_{n_c,n_{c'}}}h(\bbf{T}).$$
    The entropy maximizing matrix is simply \[\bbf{T}^{c,c'}=\frac{1}{n_cn_{c'}}\bbf{1}_{n_c,n_{c'}}, 
    \]
    and the cross-covariance matrices become \[
\bbf{C}^{c,c'}=\frac{1}{n_cn_{c'}}\sum_{ij}(\bbf{x}_i^c-\bbf{x}_j^{c'})(\bbf{x}_i^c-\bbf{x}_j^{c'})^T.
\]
Consequently, NTRopt-WDA reduces to LDA:
$$\max_{\bbf{P}^T\bbf{P}=I_p}\frac{\tr(\bbf{P}^T\bbf{S}_b\bbf{P})}{\tr(\bbf{P}^T\bbf{S}_w\bbf{P})}$$ 
    where 
    \begin{equation}\label{eq:LDA_Cb_Cw}
        \bbf{S}_b:=\sum_{c,c'>c}\frac{1}{n_cn_{c'}}\sum_{ij}(\bbf{x}_i^c-\bbf{x}_j^{c'})(\bbf{x}_i^c-\bbf{x}_j^{c'})^T
        \quad \mbox{and} \quad
        \bbf{S}_w:=\sum_{c}\frac{1}{n_c^2}\sum_{ij}(\bbf{x}_i^c-\bbf{x}_j^{c})(\bbf{x}_i^c-\bbf{x}_j^{c})^T.
    \end{equation}
    The between and within cross-covariance matrices $\bbf{S}_b$ and $\bbf{S}_w$, respectively, indicate that all relations between data points carry equal weight. Therefore, LDA strictly enforces global relations among data points.
    
    \item Unlike other linear DR methods that are strictly based on either global relations (such as LDA) or local relations (such as Large Margin Nearest Neighbor \cite{weinberger2009distance}) of the data points, WDA can dynamically consider both global and local relations by varying the regularization parameter $\lambda$ in the regularized Wasserstein distances. 
    
    In the regularized Wasserstein distance~\eqref{eq:weighted_sum}, the transportation weight ${T}_{ij}^{c,c'}(\bbf{P})$ quantifies the importance of the distance $\|\bbf{P}^T\bbf{x}_i^c-\bbf{P}^T\bbf{x}_j^{c'}\|_2^2$, i.e., the Euclidean distance between the projected data point $\bbf{x}_i^c$ in class $c$ and the projected data point $\bbf{x}_j^{c'}$ in class $c'$. By varying the strength of $\lambda$, the regularized Wasserstein distance is able to control the value of the transportation weights and thus the relations between data points. Small $\lambda$ puts more emphasis on the global coherency and large $\lambda$ puts more emphasis on the local structure of the class manifold.
    
    Figure~\ref{fig:WDA_T_example} illustrates the impact that $\lambda$ has on the transportation weights. Shown in the figure is a synthetic dataset consisting of a class that displays a half-moon shape and another class that is a bi-modal Gaussian distribution with a mode on each side of the half-moon shape.
    The transportation weights are represented as the edges between the data points such that their magnitudes are reflected by the visibility of the edges. The left column shows the inter-class relations and the right column shows the intra-class relations. Each row corresponds to a different $\lambda$ value, with the first row corresponding to $\lambda=0.1$, the second row to $\lambda=0.5$, and the third row to $\lambda=1$.
    For both inter-class and intra-class relations, we observe that as $\lambda$ increases from $\lambda=0.1$ to $\lambda=1$ the structure of the class manifold shifts from globality to locality. For instance, when $\lambda=0.1$, if we focus on one of the red data points on the top mode we observe that it has strong inter-class relations with almost all the blue points and strong intra-class relations with almost all the red points. However, when $\lambda=1$, the same red data point has strong inter-class relations only with the blue points that are close-by and strong intra-class relations only with the red points that are in the top modal.

    The impact of $\lambda$ on the global and local relations between the projected data points can be explained by examining the formulation of the transportation weights and the sum constraint imposed on them. 
    Suppose that we have three data points $\bbf{x}_i^c, \bbf{x}_j^{c'}, \bbf{x}_k^{c'}$ such that $\bbf{P}^T\bbf{x}_i^{c}$ is closer to $\bbf{P}^T\bbf{x}_j^{c'}$ than it is to $\bbf{P}^T\bbf{x}_k^{c'}$, i.e,
    \begin{equation}\label{eq:proj_data_rel}
        \|\bbf{P}^T\bbf{x}_i^c - \bbf{P}^T\bbf{x}_j^{c'}\|_2^2 < \|\bbf{P}^T\bbf{x}_i^c - \bbf{P}^T\bbf{x}_k^{c'}\|_2^2.
    \end{equation}
    According to Theorem~\ref{thm:Tlambda}, the transportation weights $T_{ij}^{c,c'}(\bbf{P})$ and $T_{ik}^{c,c'}(\bbf{P})$ between $\bbf{P}^T\bbf{x}_i^c$ and $\bbf{P}^T\bbf{x}_j^{c'}$ and between $\bbf{P}^T\bbf{x}_i^c$ and $\bbf{P}^T\bbf{x}_k^{c'}$, respectively, are equal to
    \begin{equation}\label{eq:Tij}
        T_{ij}^{c,c'}(\bbf{P})=u_i^{c,c'}e^{-\lambda\|\bbf{P}^T\bbf{x}_i^c-\bbf{P}^T\bbf{x}_j^{c'}\|_2^2}v_j^{c,c'}
        \quad \mbox{and} \quad
        T_{ik}^{c,c'}(\bbf{P})=u_i^{c,c'}e^{-\lambda\|\bbf{P}^T\bbf{x}_i^c-\bbf{P}^T\bbf{x}_k^{c'}\|_2^2}v_k^{c,c'},
    \end{equation}
    where $u_i^{c,c'},v_j^{c,c'},v_k^{c,c'}$ are found by an algorithm for matrix balancing problems such as the Sinkhorn-Knopp (SK) algorithm or the accelerated SK (Acc-SK) algorithm.
    According to \eqref{eq:Tij}, the condition \eqref{eq:proj_data_rel} suggests that $T_{ij}^{c,c'}(\bbf{P}) > T_{ik}^{c,c'}(\bbf{P})$, with the value of $\lambda$ determining the magnitude difference between $T_{ij}^{c,c'}(\bbf{P})$ and $T_{ik}^{c,c'}(\bbf{P})$. 
    Specifically, as $\lambda$ increases, the disparity in size between $T_{ij}^{c,c'}(\bbf{P})$ and $T_{ik}^{c,c'}(\bbf{P})$ also increases. This indicates that the relationship between closely-distanced data points is emphasized more than the relationship between widely-distanced data points.
    Furthermore, the sum constraint on the OT matrix $T^{c,c'}(\bbf{P})$, expressed as
    \begin{equation}\label{eq:stoc_const}
        \bbf{T}^{c,c'}(\bbf{P})\bbf{1}_{n_{c'}} = \frac{1}{n_c}\bbf{1}_{n_c}\quad\text{and}\quad \bbf{T}^{c,c'}(\bbf{P})^T\bbf{1}_{n_c} = \frac{1}{n_{c'}}\bbf{1}_{n_{c'}},
    \end{equation}
    implies that satisfying the constraint~\eqref{eq:stoc_const} requires an inverse relationship between transportation weights. 
    When one transportation weight becomes small, the other transportation weights must become large.
    In other words, for large $\lambda$, the weaker data relations between distantly located data points contribute to the stronger data relations between closely located data points. This leads to a more pronounced locality in the relationships within the class manifold.

\end{itemize}

\begin{figure}[t]
    \centering
    \includegraphics[width=0.6\textwidth]{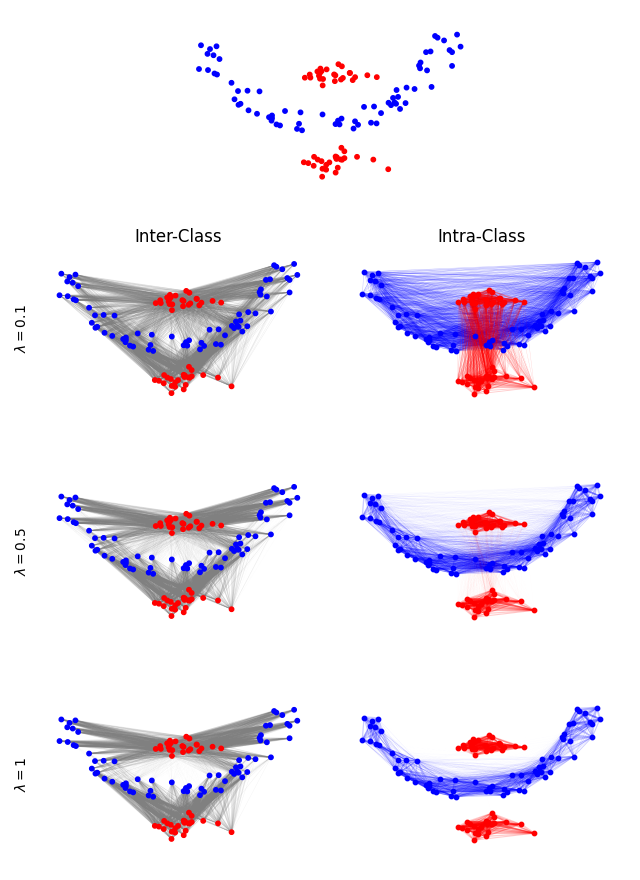}
    \caption{Illustration of globality and locality relations of two class datasets as the regularization parameter takes on the values $\lambda=0.1,0.5,1$. The first plot displays the shape of two classes, and the following left column and the following right column illustrate inter-class relations and intra-class relations, respectively, with each row corresponding to different $\lambda$.
    } \label{fig:WDA_T_example}
\end{figure}


\section{Existing Algorithms}\label{sec:exist}
To the best of our knowledge, there are two algorithms that are specifically designed for solving the NTRopt-WDA~\eqref{eq:WDA_tropt}. In this section, we provide a summary of these algorithms and point out their shortcomings.

\paragraph{WDA-gd.}
In \cite{flamary2018wasserstein}, NTRopt-WDA~\eqref{eq:WDA_tropt} is solved using the projected gradient descent method. The derivative of the OT matrix is computed as the derivative of the matrix $\bbf{T}^L(\bbf{P})=\mathcal{D}(\bbf{u}^L)\bbf{K}\mathcal{D}(\bbf{v}^L)$ obtained from a preset $L^{\mbox{th}}$ iteration of SK algorithm (see Section~\ref{subsec:T} for the discussion on SK algorithm). A recursion occurs in computing the derivative of $\bbf{T}^L(\bbf{P})$ as $\bbf{u}^L$ depends on $\bbf{v}^L$ which in turn depends on $\bbf{u}^{L-1}$. An automatic differentiation is used to compute the derivative of $\bbf{T}^L(\bbf{P})$ in order to utilize the recursion scheme of SK algorithm. This derivative-based algorithm is referred to as WDA-gd.
However, WDA-gd is subject to some practical shortcomings.
\begin{itemize}
    \item A large number of SK iterations may be needed to obtain the OT matrices and their derivatives (see Example~\ref{eg:synth_MB}). 
    If the iteration number $L$ is not large enough, both the OT matrices and their derivatives may be inaccurate.
    \item Computing the derivative of the objective function of NTRopt-WDA comes with a heavy cost that scales quadratically in the dimension size, the number of data points, and the number of data classes. Assuming that each class has $n$ number of data points, at each iteration of WDA-gd, the costs of computing the derivatives of the OT matrices from automatic differentiation are $O(Lpn^2d^2)$, where $L$ is the number of SK iterations, $p$ is the size of subspace dimension, and $d$ is the size of the original dimension. The other computation costs in the derivative of the objective function are $O(pC^2n^2d)$ where $C$ is the number of classes. Therefore, the overall cost for computing the derivative at each iteration is $O(LpC^2n^2d^2)$. 
    
\end{itemize}

\paragraph{WDA-eig.}
In \cite{liu2020ratio}, the authors proposed to approximate NTRopt-WDA~\eqref{eq:WDA_tropt} with the following ratio trace surrogate model
\begin{equation}\label{eq:WDA_rtopt}
    \max_{\bbf{P}^T\bbf{P}=I_p}\tr \bigg((\bbf{P}^T\bbf{C}_b(\bbf{P})\bbf{P})(\bbf{P}^T\bbf{C}_w(\bbf{P})\bbf{P})^{-1} \bigg),
\end{equation}
then further approximate this surrogate model using the NEPv:
\begin{equation}\label{eq:WDA_NGEPv}
\bbf{C}_b(\bbf{P})\bbf{P}=\bbf{C}_w(\bbf{P})\bbf{P}\Lambda,
\end{equation}
where $\bbf{P}$ is the eigenvector corresponding to the $p$ largest eigenvalues of  $(\bbf{C}_b(\bbf{P}),\bbf{C}_w(\bbf{P}))$.
The NEPv~\eqref{eq:WDA_NGEPv} was solved using the self-consistent field (SCF) iteration: 
\begin{equation}\label{eq:GNEPv_SCF}
    \bbf{P}_{k+1}\leftarrow \mbox{eigenvectors of the $p$ largest eigenvalues of $(\bbf{C}_b(\bbf{P_k}),\bbf{C}_w(\bbf{P_k}))$}.
\end{equation}
The approach was named WDA-eig.
There are major issues with this approach.
\begin{itemize}
    \item The OT matrices needed in $\bbf{C}_b(\bbf{P}_k)$ and $\bbf{C}_w(\bbf{P}_k)$ are computed using SK algorithm. As a result, WDA-eig suffers from the same shortcomings as WDA-gd of needing a large number of SK iterations.
    
    \item The surrogate model~\eqref{eq:WDA_rtopt} may be a poor approximation to NTRopt-WDA~\eqref{eq:WDA_tropt} (See Figure~\ref{fig:WDA_WINE_IRIS_Obj} where the surrogate model~\eqref{eq:WDA_rtopt} leads to a suboptimal solution). 
    In fact, when $p>1$ and $(\bbf{C}_b(\bbf{P}),\bbf{C}_w(\bbf{P}))$
    are constant pair $(\bbf{A},\bbf{B})$, it has been shown that the ratio trace optimization
    \begin{equation}\label{eq:RT}
        \max_{\bbf{P}^T\bbf{P}=I_p}\tr \bigg((\bbf{P}^T\bbf{A}\bbf{P})(\bbf{P}^T\bbf{B}\bbf{P})^{-1} \bigg)
    \end{equation}
    could be a poor approximation to the trace ratio optimization
    \begin{equation}\label{eq:TRopt}
        \max_{\bbf{P}^T\bbf{P}=I_p}\frac{\tr(\bbf{P}^T\bbf{A}\bbf{P})}{\tr(\bbf{P}^T\bbf{B}\bbf{P})}.
    \end{equation}
    and degrades the classification outcomes \cite{wang2007trace}. 
    
    \item There is no proper quantification for the quality of the surrogate model~\eqref{eq:WDA_rtopt} to the NEPv~\eqref{eq:WDA_NGEPv}. It is not known if the solution $\bbf{P}$ of the model~\eqref{eq:WDA_rtopt} corresponds to the $p$ largest eigenvectors of the NEPv~\eqref{eq:WDA_NGEPv}. Therefore, the converged projection from the SCF iteration~\eqref{eq:GNEPv_SCF} may not be the solution to the surrogate model~\eqref{eq:WDA_rtopt}.
\end{itemize}


\section{A Bi-Level Nonlinear Eigenvector Algorithm}\label{sec:alg}

In this section, we present a new algorithm, called WDA-nepv, to address the shortcomings of WDA-gd and WDA-eig. 
Unlike WDA-gd, WDA-nepv eliminates the need for computing the derivatives of the objective function $f(\bbf{P})$ in \eqref{eq:WDA_tropt}.
Meanwhile, unlike WDA-eig, WDA-nepv directly tackles the objective function $f(\bbf{P})$ without using a surrogate function.
Therefore, WDA-nepv is {\em derivative-free} and {\em surrogate-model-free}. 
WDA-nepv fully leverages the bi-level structure of NTRopt-WDA~\eqref{eq:WDA_tropt} by formulating both the inner and outer optimizations as NEPvs. This approach presents a unified framework for addressing NTRopt-WDA~\eqref{eq:WDA_tropt}.

\paragraph{Algorithm outline.}
WDA-nepv follows a bi-level or inner-outer iteration scheme in which at each iteration the projection dependence on the cross-covariance matrices is fixed and the projection is updated as the solution of the resulting trace ratio optimization (TRopt). There lie two NEPvs in WDA-nepv: one for computing the OT matrices of the cross-covariance matrices and the other for computing TRopt.
\newcommand{\zapspace}{\topsep=0pt\partopsep=0pt\itemsep=0pt\parskip=0pt}
Here is an outline of WDA-nepv: 
\begin{enumerate} 
    \item Start with an initial projection $\bbf{P}_0$ with $\bbf{P}^T_0 \bbf{P}_0 = \bbf{I}_p$.
    \item At the $k$th iteration:
    \begin{enumerate} 
        \item Compute the OT matrices $\bbf{T}^{c,c'}(\bbf{P}_k)$ and $\bbf{T}^{c,c}(\bbf{P}_k)$ by an NEPv (Sec.~\ref{subsec:T}).
        \item Use the OT matrices $\bbf{T}^{c,c'}(\bbf{P}_k)$ and $\bbf{T}^{c,c}(\bbf{P}_k)$ to form between and within cross-covariances $\bbf{C}_b(\bbf{P}_k)$ and $\bbf{C}_w(\bbf{P}_k)$ by level-3 BLAS (see Sec.~\ref{subsec:cov}).
        \item Compute $\bbf{P}_{k+1}$:
        \begin{equation}\label{eq:FP_TR}
            \bbf{P}_{k+1} = \argmax_{\bbf{P}^T\bbf{P}=I_p}
            \frac{\tr(\bbf{P}^T\bbf{C}_b(\bbf{P}_k)\bbf{P})}
            {\tr(\bbf{P}^T\bbf{C}_w(\bbf{P}_k)\bbf{P})}
        \end{equation}
        by another NEPv for TRopt (see Sec.~\ref{subsec:second_itr});
    \end{enumerate}
    \item Terminate when $\bbf{P}_k$ and $\bbf{P}_{k+1}$ are sufficiently close (see remarks of Algorithm~\ref{alg:BI}).
\end{enumerate}
We solve both NEPvs using the self-consistent field (SCF) iteration. SCF iteration was originally implemented in solving Kohn-Sham equation arising in physics and quantum-chemistry \cite{cances2001self}, and remains an active research topic and a popular method of choice for solving NEPvs \cite{bai2018robust,bai2020optimal,cai2018eigenvector}.


\subsection{Computing OT Matrices by NEPv}\label{subsec:T}

In NTRopt-WDA~\eqref{eq:WDA_tropt}, a key computation step involves computing the OT matrices defined in \eqref{eq:WDA_otp}. The OT matrices take an integral part in quantifying the global and local relations of data points. Therefore, accurate and efficient computations of OT matrices are crucial. According to Theorem~\ref{thm:Tlambda}, the computation of the OT matrices can be recast as a matrix balancing problem. In this section, we discuss Sinkhorn-Knopp (SK) algorithm, one of the most well-known algorithms for the matrix balancing problem, followed by its accelerated variant (named Acc-SK) that can converge in a smaller number of iterations. In contrast to WDA-gd and WDA-eig which use SK algorithm directly, we use Acc-SK to compute the OT matrices. We conclude the section with two examples to demonstrate the efficiency and accuracy of Acc-SK in comparison to SK.

\paragraph{Matrix balancing problem.}
For the sake of simplicity, we denote the OT matrices $\bbf{T}^{c,c'}(\bbf{P}_k)$ and Euclidean distance matrix $\bbf{M}_{\bbf{P}^T \bbf{X}^c,\bbf{P}^T \bbf{X}^{c'}}$ as $\bbf{T}^\lambda\in\RR^{n\times m}_+$ and $\bbf{M}\in\RR_+^{n\times m}$, respectively. 
According to Theorem~\ref{thm:Tlambda}, 
$$\bbf{T}^\lambda=\mathcal{D}(\bbf{u})\bbf{K}\mathcal{D}(\bbf{v})$$
where $\bbf{K}:=e^{-\lambda\bbf{M}}$ is an element-wise exponential, and $\bbf{u}\in\RR_+^n$ and $\bbf{v}\in\RR_+^m$ satisfy the relation
\begin{eqnarray}\label{eq:OT_sol}
\begin{cases}
\mathcal{D}(\bbf{u})\bbf{K}\mathcal{D}(\bbf{v})\bbf{1}_m=\frac{1}{n}\bbf{1}_n\\
\mathcal{D}(\bbf{v})\bbf{K}^T\mathcal{D}(\bbf{u})\bbf{1}_n=\frac{1}{m}\bbf{1}_m.
\end{cases}
\end{eqnarray}
Computing $\bbf{u}\in\RR_+^n$ and $\bbf{v}\in\RR_+^m$ from $\bbf{K}$ to satisfy the relation~\eqref{eq:OT_sol} is known as the matrix balancing problem.

There is a vast literature on the matrix balancing problem and one of its earliest works can be dated back as far as 1937 \cite{kruithof1937telefoonverkeersrekening} where the problem arose in the calculation of traffic flow. Its applications lie in many different areas such as balancing the matrix to improve the sensitivity of the eigenvalue problem \cite{parlett1971balancing} or balancing the matrix pencil to improve the sensitivity of the generalized eigenvalue problem \cite{ward1981balancing, lemonnier2006balancing, dopico2022diagonal} in order to compute more accurate eigenvalues, and in optimal transport \cite{cuturi2013sinkhorn} to compute the distance between probability vectors. See \cite{idel2016review} for further discussions on the applications and the historical remarks of the matrix balancing problems.

The most notable work on the matrix balancing problem is Sinkhorn's paper in 1964 \cite{sinkhorn1964relationship} that showed the matrix balancing problem is solvable for a positive square matrix. In 1967 \cite{sinkhorn1967diagonal}, Sinkhorn further extended his result to prove the existence of a matrix with prescribed row and column sums from a positive rectangular matrix, and with Knopp \cite{sinkhorn1967concerning}, derived conditions for the existence of a doubly stochastic matrix from a non-negative square matrix.

\paragraph{Sinkhorn-Knopp (SK) algorithm.}
Sinkhorn-Knopp (SK) algorithm \cite{sinkhorn1964relationship,sinkhorn1967diagonal,sinkhorn1967concerning}, also known as RAS method or Bregman's balancing method, is one of the most well-known algorithms for the matrix balancing problem. SK iteratively scales the matrix involving only matrix-vector multiplications, and converges linearly \cite{kalantari1996complexity,franklin1989scaling, knight2008sinkhorn}. Applied to computing the OT matrix $\bbf{T}^\lambda$, SK obtains the vectors $\bbf{v}$ and $\bbf{u}$ satisfying~\eqref{eq:OT_sol} by an alternating updating scheme on the matrix $\bbf{K}$:
\begin{eqnarray}\label{eq:Sink}
\begin{cases}
\bbf{v}_{k+1}=\frac{1}{m}\bbf{1}_m./(\bbf{K}^T\bbf{u}_{k}) \\
\bbf{u}_{k+1}=\frac{1}{n}\bbf{1}_n./(\bbf{K}\bbf{v}_{k+1}).
\end{cases}
\end{eqnarray}

\paragraph{Reformulation of SK as nonlinear mapping and NEPv.}
While simple, the SK algorithm is subject to slow convergence \cite{aristodemo2020accelerating, knight2008sinkhorn}. There are variants of the SK algorithm for acceleration. They include a different updating scheme for scaling the matrix \cite{parlett1982methods}, an inner-outer iteration algorithm based on Newton's method \cite{knight2013fast}, a greedy Sinkhorn algorithm called Greenkhorn algorithm \cite{alaya2019screening,altschuler2017near}, and the SCF iteration for solving the NEPv formulation \cite{aristodemo2020accelerating}.

In our algorithm WDA-nepv, we consider accelerating the SK algorithm by solving the NEPv formulation. We note that the SK iteration~\eqref{eq:Sink} for obtaining the vector $\bbf{v}$
can be rewritten as the following iteration 
\begin{equation}\label{eq:fixed_point_itr}
    \bbf{v}_{k+1}=R(\bbf{v}_k),
\end{equation}
where
\begin{equation}\label{eq:R}
   R(\bbf{v}) :=\mathcal{D}^{-1}\bigg(\bbf{K}^T\mathcal{D}^{-1}\bigg(\bbf{K}\bbf{v}\bigg)\frac{\bbf{1}_n}{n}\bigg)\frac{\bbf{1}_m}{m}
\end{equation}
is a nonlinear mapping. 
$R$ is a contraction mapping on the space of non-negative vectors~\cite{brualdi1966diagonal}. Thus, the fixed point of $R$ exists and is unique.

In the following, we show that the fixed point $\bbf{v}$ of the mapping $R$ is an eigenvector corresponding to the largest eigenvalue of $J_R(\bbf{v})\in\RR^{m\times m}$, the Jacobian of the mapping $R$. First, we have the following theorem due to \cite[Theorem 5]{aristodemo2020accelerating}.\footnote{Although only positive square matrices were subject to discussion in \cite{aristodemo2020accelerating}, we can extend their results to positive rectangular matrices.}

\begin{theorem}\label{thm:Jac_R}
For a positive vector $\bbf{v}$, 
it holds
\begin{equation}\label{eq:Jac_R}
    R(\bbf{v}) = J_R(\bbf{v}) \bbf{v},
\end{equation}
where $J_R$ is the Jacobian of the
mapping $R$ defined as in \eqref{eq:R}.
\end{theorem}

\begin{proof}
The equation~\eqref{eq:R} can be written as the composition of mappings
\begin{eqnarray}\label{eq:R_vk}
  R(\bbf{v})=\frac{n}{m}U(\bbf{K}^TS(\bbf{v}))
\end{eqnarray}
where $U$ and $S$ are defined as
\begin{equation}\label{eq:US}
    U(\bbf{v})=\bbf{1}./\bbf{v} \quad \mbox{and} \quad 
    S(\bbf{v})=U(\bbf{K}\bbf{v}).
\end{equation}
Note that we define $U$ as a component-wise reciprocal operator for vectors of any size. In particular, since $\bbf{K}$ is $n\times m$, $U$ in the operator $S$ operates on the vector of size $n$.

The Jacobian of $U$ is
\begin{align}\label{eq:JU}
    J_U(\bbf{v})&=-\mathcal{D}^{-2}(\bbf{v})
    =-\mathcal{D}^2(\bbf{1}./\bbf{v})
    =-\mathcal{D}^2(U(\bbf{v})),
\end{align}
where the last equality holds by~\eqref{eq:US}. 
The Jacobian of $S$ is
\begin{align}\label{eq:JS}
    J_S(\bbf{v})=-\mathcal{D}^{-2}(\bbf{K}\bbf{v})\bbf{K}
    =-\mathcal{D}^2(\bbf{1}_n./(\bbf{K}\bbf{v}))\bbf{K}
    =-\mathcal{D}^2(U(\bbf{K}\bbf{v}))\bbf{K}
    =-\mathcal{D}^2(S(\bbf{v}))\bbf{K},
\end{align}
where the last two equalities hold by~\eqref{eq:US}.

By the chain rule, the Jacobian matrix of the mapping $R$ is
\begin{align*}
    J_R(\bbf{v})&=J_{\frac{n}{m}U\bbf{K}^TS}(\bbf{v})
    =\frac{n}{m}J_U(\bbf{K}^TS(\bbf{v}))\cdot \bbf{K}^T\cdot J_S(\bbf{v})
\end{align*}
Utilizing the results~\eqref{eq:JU}, \eqref{eq:JS} and substituting the mappings $R$~\eqref{eq:R_vk}, $J_R(\bbf{v})$ is further simplified as
\begin{align}\label{eq:Jac}
    J_R(\bbf{v})&=\frac{n}{m}[-\mathcal{D}^2(U(\bbf{K}^TS(\bbf{v})))][\bbf{K}^T][-\mathcal{D}^2(S(\bbf{v}))\bbf{K}]\nonumber\\
    &=\frac{m}{n}\mathcal{D}^2(R(\bbf{v}))\bbf{K}^T\mathcal{D}^2(S(\bbf{v}))\bbf{K}
\end{align}
Then,
\begin{align*}
    J_R(\bbf{v})\bbf{v}&=\frac{m}{n}\mathcal{D}^2(R(\bbf{v}))\bbf{K}^T\mathcal{D}^2(S(\bbf{v}))\bbf{K}\bbf{v}
    =\frac{m}{n}\mathcal{D}^2(R(\bbf{v}))\bbf{K}^T\mathcal{D}^2(U(\bbf{K}\bbf{v}))\bbf{K}\bbf{v}\\
    &=\frac{m}{n}\mathcal{D}^2(R(\bbf{v}))\bbf{K}^T\mathcal{D}^{-2}(\bbf{K}\bbf{v})\bbf{K}\bbf{v}
    =\frac{m}{n}\mathcal{D}^2(R(\bbf{v}))\bbf{K}^T\mathcal{D}^{-1}(\bbf{K}\bbf{v})\bbf{1}_n\\
    &=\frac{m}{n}\mathcal{D}^2(R(\bbf{v}))\bbf{K}^T\mathcal{D}(U(\bbf{K}\bbf{v}))\bbf{1}_n
    =\frac{m}{n}\mathcal{D}^2(R(\bbf{v}))\bbf{K}^TS(\bbf{v})\\
    &=\frac{n}{m}\mathcal{D}^{-2}(\bbf{K}^TS(\bbf{v}))\bbf{K}^TS(\bbf{v})
    =\frac{n}{m}\mathcal{D}^{-1}(\bbf{K}^TS(\bbf{v}))\bbf{1}_m\\
    &=\frac{n}{m}\mathcal{D}(U(\bbf{K}^TS(\bbf{v})))\bbf{1}_m
    =\mathcal{D}(R(\bbf{v}))\bbf{1}_m
    =R(\bbf{v})
\end{align*}
\end{proof}

By Theorem~\ref{thm:Jac_R}, the fixed point $\bbf{v}$ of the mapping $R$ satisfies
\begin{equation}\label{eq:fp_eig}
    J_R(\bbf{v})\bbf{v}=\bbf{v},
\end{equation}
i.e., $\bbf{v}$ is an eigenvector of $J_R(\bbf{v})$ corresponding to the eigenvalue $\mu=1$. Note that the fixed point $\bbf{v}$ is strictly positive and the exponential Euclidean distance matrix $\bbf{K}$ is strictly positive. Consequently, it is clear from the equation~\eqref{eq:Jac} that $J_R(\bbf{v})$ is a positive matrix. Therefore, by the Perron-Frobenius theorem \cite{golub2013matrix}, $\bbf{v}$ must be the eigenvector corresponding to the largest eigenvalue of the NEPv:
\begin{equation}\label{eq:NEPv}
    J_R(\bbf{v})\bbf{v}=\mu\bbf{v}
\end{equation}

\paragraph{Accelerated Sinkhorn-Knopp (Acc-SK) algorithm.}
The above discussion indicates that the SK algorithm~\eqref{eq:Sink} is equivalent to finding the eigenvector $\bbf{v}$ corresponding to the largest eigenvalue of $J_R(\bbf{v})$. 
To accelerate the SK iteration, we can apply the following SCF iteration for solving the NEPv~\eqref{eq:NEPv}: 
\begin{equation}\label{eq:Jac_SCF}
    \bbf{v}_{k+1}\leftarrow 
    \mbox{eigenvector of $\mu_{\max}(J_R(\bbf{v}_k))$.}
\end{equation}
The SCF iteration~\eqref{eq:Jac_SCF} is referred to as Acc-SK. 
Once Acc-SK converges to $\bbf{v}$, $\bbf{u}$ is computed by an extra update \eqref{eq:Sink}, i.e.,
\[
\bbf{u}=\frac{1}{n}\bbf{1}_n./(\bbf{K}\bbf{v}).
\]
In summary, Acc-SK for computing the OT matrix $\bbf{T}^\lambda$ is shown in Algorithm~\ref{alg:T}.
The eigenvalue problem in line 4 of Algorithm~\ref{alg:T} can be efficiently solved by a Krylov-subspace eigensolver such as the implicitly restarted Arnoldi method \cite{sorensen1997implicitly}.
The stopping criteria in line 5 checks whether the distance between $\bbf{v}_{k+1}$ and $\bbf{v}_k$ is smaller than a preset tolerance $tol$.

\begin{algorithm}[h]
\caption{Computation of $\bbf{T}^\lambda$ via NEPv} \label{alg:T}
\textbf{Input}: Matrix $\bbf{K}\in\RR_+^{n\times m}$, tolerance tol\\
\textbf{Output}:  the OT matrix $\bbf{T}^\lambda$ in \eqref{eq:T}
    \begin{algorithmic}[1]
    \STATE Create a starting $\bbf{v}_0>0$
    \FOR{$k=0,1,\ldots$}
    \STATE Set $J_k=J_R(\bbf{v}_k)$
    \STATE Set $\bbf{v}_{k+1}$ as the eigenvector of the largest eigenvalue of $J_k$
    \IF{$d(\bbf{v}_{k+1},\bbf{v}_k)< tol$}
    \STATE return $\bbf{v}=\bbf{v}_{k+1}$
    \ENDIF
    \ENDFOR
    \STATE Compute $\bbf{u}=\frac{1}{n}\bbf{1}_n./(\bbf{K}\bbf{v})$
    \STATE Return $\bbf{T}^\lambda=\mathcal{D}(\bbf{u})\bbf{K}\mathcal{D}(\bbf{v})$
    \end{algorithmic}
\end{algorithm}

\paragraph{Advantages of Acc-SK.}
Acc-SK significantly improves over SK when the magnitude of the components of an input matrix is small \cite{aristodemo2020accelerating}.
In the context of NTRopt-WDA~\eqref{eq:WDA_tropt}, this situation can occur when the regularization parameter $\lambda$ is sufficiently large, i.e., when the local relationships between data points are emphasized.
We illustrate the acceleration of Acc-SK over SK in the following two examples: the first example for general synthetic matrices and the second example for computing the OT matrices of NTRopt-WDA~\eqref{eq:WDA_tropt}. For both SK and Acc-SK, the starting vector is taken as 
$\bbf{v}_0=\frac{1}{m}\bbf{1}_m$ and the convergence behavior is displayed by the error $\|\bbf{v}_{k+1}-\bbf{v}_k\|_2$. We show that while SK is subject to slow convergence, Acc-SK converges efficiently in a small number of iterations.

\begin{example}\label{eg:gen_synth_MB}
{\rm  
Consider the following synthetic matrices
\begin{align*}
    \bbf{K}_1 =
    \begin{bmatrix}
    1 & \epsilon \\
    1 & 1
    \end{bmatrix},\quad
    \bbf{K}_2 = 
    \begin{bmatrix}
    1 & \epsilon \\
    1 & 1 \\
    1 & 1
    \end{bmatrix}
\end{align*}
where $\epsilon$ is a small positive number. 
The convergence behaviors of SK and Acc-SK for the matrices $\bbf{K}_1$ and $\bbf{K}_2$ with $\epsilon=10^{-8}$ in Figure~\ref{fig:SKvsAccSK} shows that SK is unable to converge within 50 iterations for matrix $\bbf{K}_1$ and that while it converges for the matrix $\bbf{K}_2$, it does so slowly. Meanwhile, Acc-SK converges in around 10 iterations for both matrices.

\begin{figure}[h]
\centering
\subfloat[Matrix $\bbf{K}_1$]{{\includegraphics[width=0.45\textwidth]{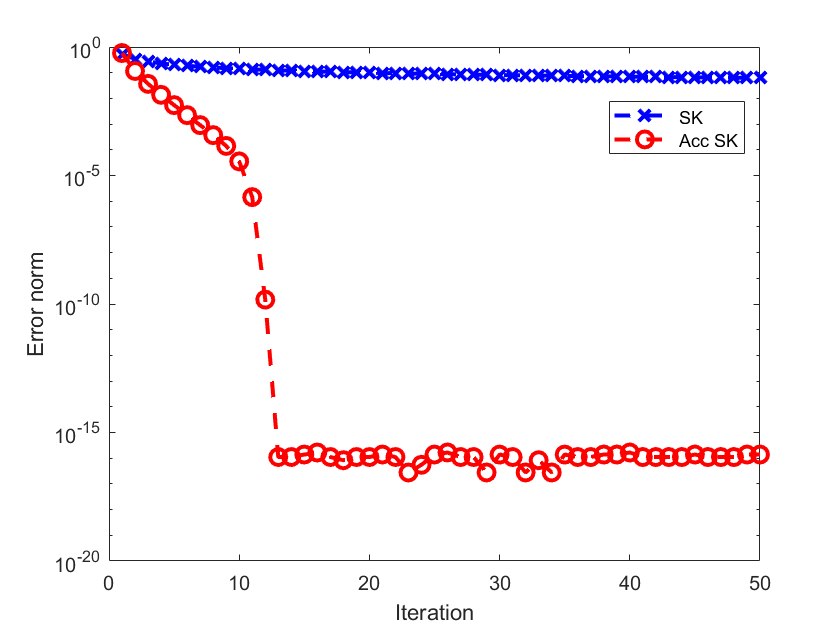} }}%
\quad
\subfloat[Matrix $\bbf{K}_2$]{{\includegraphics[width=0.45\textwidth]{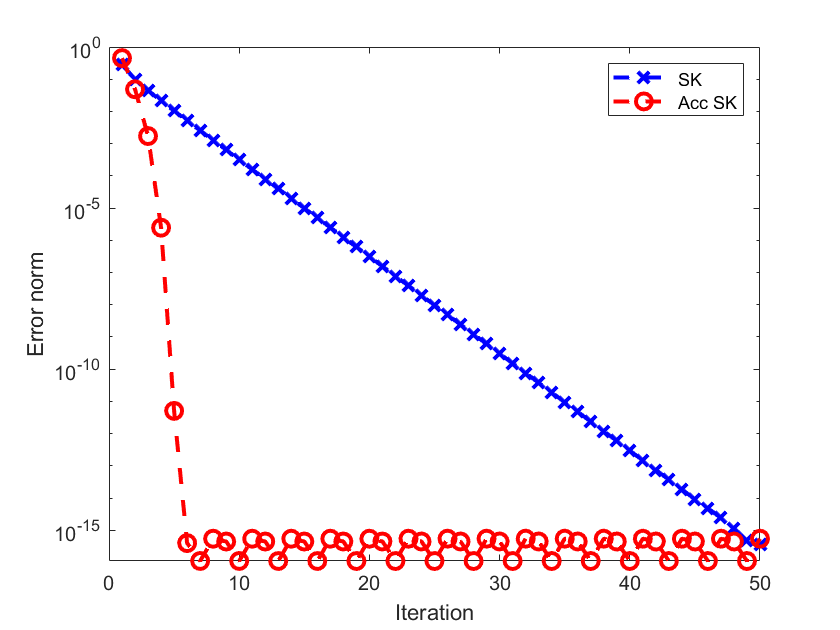} }}%
\caption{Convergence behaviors of SK and Acc-SK on matrices $\bbf{K}_1$ and $\bbf{K}_2$. Each plot displays the error $\|\bbf{v}_{k+1}-\bbf{v}_k\|_2$ of SK and Acc-SK.}\label{fig:SKvsAccSK}
\end{figure}

}
\end{example}

\begin{example}\label{eg:synth_MB}
{\rm
Using the synthetic dataset (see Sec.~\ref{subsec:data} for its description), we create data matrices $\bbf{X}^1\in\RR^{2\times60},\bbf{X}^2\in\RR^{2\times100}$. We choose a sufficiently large regularization parameter $\lambda=1$ to emphasize the local relations of the data points and form exponential Euclidean distance matrices $\bbf{K}^{1,2}\in\RR_+^{60\times100}$ and $\bbf{K}^{1,1}\in\RR_+^{60\times60}$. With this choice of the regularization parameter, some of the components of $\bbf{K}^{1,2}$ and $\bbf{K}^{1,1}$ are small. 
Figure~\ref{fig:SKvsAccSK_Synth1} depicts the convergence behavior.
We observe that SK exhibits slow convergence for $\bbf{K}^{1,2}$ and does not converge below the error $10^{-5}$ for $\bbf{K}^{1,1}$. On the other hand, Acc-SK displays fast and accurate convergence, converging in just two iterations for both $\bbf{K}^{1,2}$ and $\bbf{K}^{1,1}$. 

\begin{figure}[H]
\centering
\includegraphics[width=0.8\textwidth]{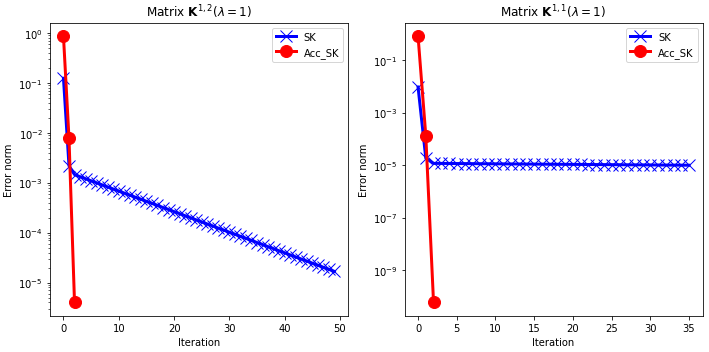}
\caption{Convergence behaviors of SK and Acc-SK on matrices $\bbf{K}^{1,2}$ and $\bbf{K}^{1,1}$. Each plot displays the error $\|\bbf{v}_{k+1}-\bbf{v}_k\|_2$ of SK and Acc-SK.} \label{fig:SKvsAccSK_Synth1}
\end{figure}

}
\end{example}

\subsection{Computation of Cross-covariance Matrices \texorpdfstring{$\bbf{C}_b(\bbf{P}_k)$}{CbPk} and \texorpdfstring{$\bbf{C}_w(\bbf{P}_k)$}{CwPk}} \label{subsec:cov}

Recall that the between and within cross-covariance matrices $\bbf{C}_b(\bbf{P}_k)$ and $\bbf{C}_w(\bbf{P}_k)$ defined in \eqref{eq:CbCw1} and \eqref{eq:CbCw2}
\begin{align*}
    \bbf{C}_b(\bbf{P}_k) &= \sum_{c,c'>c}  \sum_{ij}{T}_{ij}^{c,c'}(\bbf{P}_k)(\bbf{x}_i^c-\bbf{x}_j^{c'})(\bbf{x}_i^c-\bbf{x}_j^{c'})^T, \\
    \bbf{C}_w(\bbf{P}_k) &= \sum_c  \sum_{ij}{T}_{ij}^{c,c}(\bbf{P}_k)(\bbf{x}_i^c-\bbf{x}_j^{c})(\bbf{x}_i^c-\bbf{x}_j^{c})^T 
\end{align*}
are the weighted sums of the outer products of the data points, where the weights are the components of the OT matrices. The number of double sums in $\bbf{C}_b(\bbf{P}_k)$ and $\bbf{C}_w(\bbf{P}_k)$ are 
\[
n_b:=\sum_{c,c'>c}^{C}n_cn_{c'}
\quad \mbox{and} \quad 
n_w:=\sum_{c}^{C}n_c^2,
\]
respectively. Overall, the number of double sums grow quadratically in the number of classes or data points.

A straightforward implementation for computing these cross-covariance matrices is by a double for-loop with $n_b$ and $n_w$ calls to level-2 BLAS in forming the outer products of $\bbf{C}_b(\bbf{P}_k)$ and $\bbf{C}_w(\bbf{P}_k)$, respectively. 
However, memory access is the bottleneck on modern computing platforms.
Unlike level-2 BLAS, level-3 BLAS makes full reuse of data and avoids excessive movement of data to and from memory \cite{dongarra1990set}.
Therefore, we propose to compute the cross-covariance matrices using level-3 BLAS to improve computational efficiency.
Specifically, we reformulate the sum of the outer products as a matrix-matrix multiplication:
\begin{equation} \label{eq:CbCwBlas3}
    \bbf{C}_b(\bbf{P}_k) =\widehat{\bbf{C}}_b(\bbf{P}_k)\widehat{\bbf{C}}^T_b(\bbf{P}_k)
    \quad \mbox{and} \quad 
    \bbf{C}_w(\bbf{P}_k)=\widehat{\bbf{C}}_w(\bbf{P}_k) \widehat{\bbf{C}}^T_w(\bbf{P}_k)
\end{equation}
where 
$\bbf{\widehat{\bbf{C}}}_b(\bbf{P}_k)\in\RR^{d\times n_b}$ and ${\widehat{\bbf{C}}}_w(\bbf{P}_k)\in\RR^{d\times n_w}$ have columns 
\[
\sqrt{{T}_{ij}^{c,c'}(\bbf{P}_k)}(\bbf{x}_i^c-\bbf{x}_j^{c'})
\quad \mbox{and} \quad
\sqrt{{T}_{ij}^{c,c}(\bbf{P}_k)}(\bbf{x}_i^c-\bbf{x}_j^{c}),
\]
respectively. 
As level-3 BLAS, the matrix-matrix multiplications~\eqref{eq:CbCwBlas3} are far more efficient than forming the sum of outer products. In particular, the improvement in efficiency becomes more prominent as the data dimension gets large. For an instance of the Synthetic dataset (see Section~\ref{sec:num} for its description) with $d=2000$, the running time of the sum of outer products took 127 seconds while the level-3 BLAS formulation only took 0.63 seconds.

We note that efficient computations of the cross-covariance matrices using level-3 BLAS is a novelty of this work. In WDA-gd, cross-covariance matrices are not formed explicitly. Instead, each regularized Wasserstein distance is computed directly as the weighted sum~\eqref{eq:weighted_sum} and the reciprocal of the objective function~\eqref{eq:WDA} is an input in Pymanopt \cite{JMLR:v17:16-177}, 
an optimizer on matrix manifold. In WDA-eig, the cross-covariance matrices are implemented directly from their double sum formulations using tensor operations.

\subsection{Solving Trace Ratio Optimization by NEPv}\label{subsec:second_itr}

Now let us consider the outer iteration for solving the TRopt~\eqref{eq:FP_TR}.
For notation convenience, set
\[
\bbf{A} = \bbf{C}_b(\bbf{P}_k), \quad
\bbf{B} = \bbf{C}_w(\bbf{P}_k).
\]
Note that by \eqref{eq:CbCwBlas3}, $\bbf{A}$ and $\bbf{B}$ are symmetric positive-definite. Furthermore, given that the number of data points is typically larger than the dimension size, we can safely assume that $\bbf{A}$ and $\bbf{B}$ are positive definite.
Then, the TRopt~\eqref{eq:FP_TR} is equivalent to
\begin{equation}\label{eq:TR}
\max_{\bbf{P}^T\bbf{P}=I_p}\left\{ q(\bbf{P}) :=\frac{\tr(\bbf{P}^T\bbf{A}\bbf{P})}{\tr(\bbf{P}^T\bbf{B}\bbf{P})} \right\}. 
\end{equation}
TRopt~\eqref{eq:TR} has been well-studied in the literature \cite{wang2007trace,ngo2010trace,zhang2014note}. It is known that the solution $\bbf{P}$ of \eqref{eq:TR} solves the following NEPv:
\begin{equation}\label{eq:TR_NEPv}
    \bigg(\bbf{A}-q(\bbf{P})\bbf{B} \bigg)\bbf{P}=\bbf{P}\Lambda_{\bbf{P}}
\end{equation}
where $\bbf{P}$ is an orthonormal eigenbasis corresponding to 
the $p$ largest eigenvalues of $\bbf{A}-q(\bbf{P})\bbf{B}$ \cite{zhang2014note}.

The gateway algorithm for solving the NEPv~\eqref{eq:TR_NEPv} is again SCF iteration:
\begin{equation}\label{eq:TR_SCF}
    \bbf{P}_{j+1}\leftarrow \mbox{eigenvectors of the $p$ largest eigenvalues of }\bbf{A}-q(\bbf{P}_j)\bbf{B}.
\end{equation}
The SCF iteration \eqref{eq:TR_SCF} is monotonic, globally convergent to the global maximizer for any initial projection $\bbf{P}_0$, and has a local quadratic convergence \cite{cai2018eigenvector}. 

In summary, the algorithm for solving TRopt~\eqref{eq:TR} is shown in  Algorithm~\ref{alg:TR}. 
The stopping criteria $d(\bbf{P}_{j+1},\bbf{P}_j)< tol$ in line 5 represents the distance between 
$\bbf{P}_{j+1}$ and $\bbf{P}_j$, namely, the largest principal angle of the subspaces spanned by the matrices \cite{golub2013matrix}.

\begin{algorithm}[h]
    \caption{Solver for the TRopt~\eqref{eq:TR}} \label{alg:TR}
   \textbf{Input}: Matrices $\bbf{A},\bbf{B}\in\RR^{d\times d}$, 
    tolerance $\mbox{tol}$ \\ 
    \textbf{Output}: solution $\bbf{P}$ to \eqref{eq:TR}
    \begin{algorithmic}[1]
    \STATE Create a starting $\bbf{P}_0$ with $\bbf{P}_0^T\bbf{P}_0=I_p$
    \FOR{$j=0,1,\ldots$}
    
    \STATE Set $\bbf{H}(\bbf{P}_j)=\bbf{A}-q(\bbf{P}_j)\bbf{B}$ where $q(\bbf{P}_j)=\frac{\tr(\bbf{P}_j^T\bbf{A}\bbf{P}_j)}{\tr(\bbf{P}_j^T\bbf{B}\bbf{P}_j)}$


    \STATE Set $\bbf{P}_{j+1}$ as the eigenvectors of the $p$ largest eigenvalues of $\bbf{H}(\bbf{P}_j)$
    
    \IF{$d(\bbf{P}_{j+1},d(\bbf{P}_j))<tol$}
    
    \STATE return $\bbf{P}=\bbf{P}_{j+1}$
    
    \ENDIF
    
    \ENDFOR
    \end{algorithmic}
\end{algorithm}

\subsection{WDA-nepv}

Combining Algorithm~\ref{alg:T} and Algorithm~\ref{alg:TR}, WDA-nepv is presented in Algorithm~\ref{alg:BI}.
A possible choice of the initial projection $\bbf{P}_0$ 
is using a random orthogonal projection or using the projection obtained as the solution to PCA or as the solution to LDA.
In line 4, the initial for Algorithm~\ref{alg:T} can be chosen as the uniformly distributed probability vector, i.e., $\bbf{v}_0=\bbf{1}_{n_{c'}}/{n_{c'}}$ and $\bbf{v}_0=\bbf{1}_{n_c}/{n_c}$ for $\bbf{T}^{c,c'}(\bbf{P}_k)$ and $\bbf{T}^{c,c}(\bbf{P}_k)$, respectively.
In line 6, the initial projection for Algorithm~\ref{alg:TR} can be chosen as $\bbf{P}_k$, the current $k$th projection.
The stopping criteria $d(\bbf{P}_{k+1},\bbf{P}_k)< tol$ 
in line 8 is measured as the largest principal angle between the subspaces spanned by the matrices $\bbf{P}_{k+1}$ and $\bbf{P}_k$.

Regarding the complexity of Algorithm~\ref{alg:BI}, suppose that the numbers of data points of classes are the same, i.e., $n=n_c$ for all $c=1,2,\ldots, C$. 
Then, for one iteration of Algorithm~\ref{alg:BI}:
$O(C^2)$ many OT matrices are computed, and 
each OT matrix costs $O(n^2)$ at each iteration of Algorithm~\ref{alg:T}.
The formation of the cross-covariance matrices $\bbf{C}_b(\bbf{P}_k)$ and $\bbf{C}_w(\bbf{P}_w)$ are $O(d^2n^2)$. 
The other cost is a Krylov subspace eigensolver for SCF in Algorithm~\ref{alg:TR}, whose leading cost is $O(d^2)$ for the matrix-vector products on the $d\times d$ covariance matrices. 
In summary, WDA-nepv (Algorithm~\ref{alg:BI}) is quadratic in the number of data points $n$ and the dimension $d$.

\begin{algorithm}[t]
\caption{WDA-nepv} \label{alg:BI}
\textbf{Input}: Data matrices $\bbf{X}^1\in\RR^{d\times n_1},\ldots,
\bbf{X}^C\in\RR^{d\times n_C}$, regularization $\lambda\geq0$, tolerance tol \\
\textbf{Output}: solution $\bbf{P}$ to NTRopt-WDA \eqref{eq:WDA_tropt}
\begin{algorithmic}[1]
\STATE Create a starting $\bbf{P}_0$ with $\bbf{P}_0^T\bbf{P}_0=I_p$
\FOR{$k=0,1,\ldots$}



\STATE Compute $\bbf{K}^{c,c'}(\bbf{P}_k)\in\RR^{n_c\times n_{c'}}$ with $[\bbf{K}^{c,c'}(\bbf{P}_k)]_{ij}=e^{-\lambda\|\bbf{P}_k^T\bbf{x}_i^c-\bbf{P}_k^T\bbf{x}_j^{c'}\|_2^2}$ and $\bbf{K}^{c,c}(\bbf{P}_k)\in\RR^{n_c\times n_c}$ with $[\bbf{K}^{c,c}(\bbf{P}_k)]_{ij}=e^{-\lambda\|\bbf{P}_k^T\bbf{x}_i^c-\bbf{P}_k^T\bbf{x}_j^{c}\|_2^2}$

\STATE Compute OT matrices $\bbf{T}^{c,c'}(\bbf{P}_k)\in\RR^{n_c\times n_{c'}}$ and 
$\bbf{T}^{c,c}(\bbf{P}_k)\in\RR^{n_c\times n_c}$ by Alg.~\ref{alg:T} with $\bbf{K}^{c,c'}(\bbf{P}_k)$ and $\bbf{K}^{c,c}(\bbf{P}_k)$

\STATE 
Compute $\bbf{C}_b(\bbf{P}_k)$  and $\bbf{C}_w(\bbf{P}_k)$ 
by level-3 BLAS \eqref{eq:CbCwBlas3}

\STATE Compute $\bbf{P}_{k+1}$ by 
       Alg.~\ref{alg:TR} with $\bbf{C}_b(\bbf{P}_k)$ and 
$\bbf{C}_w(\bbf{P}_k)$ 
\IF{$d(\bbf{P}_{k+1},\bbf{P}_k)< tol$}
\STATE Return $\bbf{P}=\bbf{P}_{k+1}$
\ENDIF
\ENDFOR
\end{algorithmic}
\end{algorithm}

\subsection{Convergence of WDA-nepv}

In this section, we provide the convergence analysis of the proposed WDA-nepv. 
We first note that given $\bbf{P}_k$, the inner optimization (Algorithm~\ref{alg:T}) to compute the OT matrices $\bbf{T}^{c,c'}(\bbf{P}_k)$ is globally convergent. 
Furthermore, for applications of WDA, the desired accuracy is typically low, say at the range of $tol = 10^{-3}$ to $10^{-5}$.
Therefore, for simplicity, in our analysis we assume that the OT matrices are accurately computed so that they can be regarded as ``{\em exact}''.
Similarly, the SCF iteration for TRopts (Algorithm~\ref{alg:TR}) in the outer optimization is also globally convergent so that $\bbf{P}_{k+1}$ is accurately computed to be regarded as ``{\em exact}''. Therefore, we have
\begin{equation}\label{eq:SCFcvgd}
    \underbrace{
    \left[\bbf{C}_b(\bbf{P}_k)-\frac {\tr(\bbf{P}_{k+1}^T\bbf{C}_b(\bbf{P}_k)\bbf{P}_{k+1})}{\tr(\bbf{P}_{k+1}^T\bbf{C}_w(\bbf{P}_k)\bbf{P}_{k+1})}\,\bbf{C}_w(\bbf{P}_k)\right]
    }_{=:\widetilde{\bbf{H}}_k}
    \bbf{P}_{k+1}=\bbf{P}_{k+1}\bbf{\Lambda}_k,
\end{equation}
where the eigenvalues of $\bbf{\Lambda}_k$ consist of the $p$ largest eigenvalues of $\widetilde{\bbf{H}}_k$.

Next, we make the following so-called monotonicity assumption whose rigorous verification remains open but is numerically demonstrated throughout all our numerical tests. Subsequently, we introduce the Ky Fan Trace Theorem~\cite{fan1949theorem}, which plays a crucial role in the proof of Theorem~\ref{thm:cvg-NEPv}.

\begin{assumption} \label{assump:mono}
For orthonormal $\bbf{P},\widehat{\bbf{P}}\in\RR^{d\times p}$ and the function $f(\bbf{P})$ defined in~\eqref{eq:WDA_tropt}, if
\begin{equation}\label{eq:mono-assume}
\frac {\tr(\widehat{\bbf{P}}^T\bbf{C}_b(\bbf{P})\widehat{\bbf{P}})}{\tr(\widehat{\bbf{P}}^T\bbf{C}_w(\bbf{P})\widehat {\bbf{P}})}\ge f(\bbf{P})+\eta,
\end{equation}
where $\eta\in\RR$, then 
\begin{equation} \label{eq:mono-assume-b}
f(\widehat{\bbf{P}})\ge f(\bbf{P})+c\eta 
\end{equation}
for some constant $c>0$, independent of $\bbf{P},\widehat{\bbf{P}}$.
\end{assumption}

\begin{theorem}[Ky Fan Trace]\label{thm:kyfan}
    Let $\bbf{A}\in\RR^{d\times d}$ be a symmetric matrix whose eigenvalues are ordered as $\lambda_1(\bbf{A})\geq\lambda_2(\bbf{A})\geq\cdots\geq\lambda_d(\bbf{A})$. Then, for $1\leq p\leq d$, the sum of the $p$ largest eigenvalues of $\bbf{A}$ and the sum of the $p$ smallest eigenvalues of $\bbf{A}$ correspond to the following maximization and the minimization of the trace operator, respectively.
    \begin{align}
        \max_{\bbf{P}^T\bbf{P}=I_p}\tr(\bbf{P}^T\bbf{A}\bbf{P}) &= \sum_{i=1}^p\lambda_i(\bbf{A}), \label{eq:kyfan_max} \\
        \min_{\bbf{P}^T\bbf{P}=I_p}\tr(\bbf{P}^T\bbf{A}\bbf{P}) &= \sum_{i=d-p+1}^d\lambda_i(\bbf{A}). \label{eq:kyfan_min}
    \end{align}
    A solution $\bbf{P}$ satisfies \eqref{eq:kyfan_max} if and only if it is an orthonormal eigenbasis matrix corresponding to the $p$ largest eigenvalues of $\bbf{P}$. Similarly, a solution $\bbf{P}$ satisfies \eqref{eq:kyfan_min} if and only if it is an orthonormal eigenbasis matrix corresponding to the $p$ smallest eigenvalues of $\bbf{P}$.
\end{theorem}

The convergence of WDA-nepv is stated in the following theorem. It provides a theoretical justification of the SCF framework for solving NTRopt-WDA~\eqref{eq:WDA_tropt}.

\begin{theorem}\label{thm:cvg-NEPv}
Suppose Assumption~\eqref{assump:mono} holds and that the sum of the $p$ smallest eigenvalues of $\bbf{C}_w(\bbf{P})$ is uniformly bounded below for all orthonormal $\bbf{P}\in\RR^{d\times p}$. Similarly, suppose that the sum of the $p$ largest eigenvalues of $\bbf{C}_w(\bbf{P})$ is uniformly bounded above for all orthonormal $\bbf{P}\in\RR^{d\times p}$.
Let the sequence $\{\bbf{P}_k\}_{k=0}^{\infty}$ be generated by WDA-nepv (Algorithm~\ref{alg:BI}).
The following statements hold.
\begin{enumerate}[label=(\alph*)]
    \item $\{f(\bbf{P}_k)\}_{k=0}^{\infty}$ is monotonically increasing and convergent.
    
    \item $\{\bbf{P}_k\}_{k=0}^{\infty}$ has a convergent subsequence $\{\bbf{P}_k\}_{k\in\mathbb{I}}$, converging to $\bbf{P}_*$, and
            $$
            \lim_{k\to\infty} f(\bbf{P}_k)=f(\bbf{P}_*).
            $$
    
    \item $\bbf{P}_*$ is an orthonormal eigenbasis matrix of $\bbf{H}_*:=\bbf{H}(\bbf{P}_*)=\bbf{C}_b(\bbf{P}_*)-f(\bbf{P}_*)\,\bbf{C}_w(\bbf{P}_*)$ associated with its $p$ largest eigenvalues.
    \item If $\lambda_p(\bbf{H}_*)-\lambda_{p+1}(\bbf{H}_*)>0$, then
        \begin{equation}\label{eq:P*}
            \bbf{P}_*=\arg\max_{\bbf{P}^T\bbf{P}=I_p}\frac {\tr(\bbf{P}^T\bbf{C}_b(\bbf{P}_*)\bbf{P})}{\tr(\bbf{P}^T\bbf{C}_w(\bbf{P}_*)\bbf{P})},
        \end{equation}
        where $\lambda_i(\bbf{H}_*)$ is the $i^{\mbox{th}}$ largest eigenvalues of $\bbf{H}_*$.
\end{enumerate}
\end{theorem}

\begin{proof}
    \begin{enumerate}[label=(\alph*)]
        \item Recall that the projection $\bbf{P}_{k+1}$ is computed exactly to form an orthonormal eigenbasis matrix corresponding to the $p$ largest eigenvalues of $\widetilde{\bbf{H}}_k$. Therefore, by Theorem~\ref{thm:kyfan} $\bbf{P}_{k+1}$ is a solution to the following optimization:
        \begin{equation}\label{eq:kyfan}
            \max_{\bbf{P}^T\bbf{P}=I_p}\tr(\bbf{P}^T\widetilde{\bbf{H}}_k\bbf{P}),
        \end{equation}
        and we have from \eqref{eq:SCFcvgd} that
        \[0=\tr(\bbf{P}_{k+1}^T\widetilde{\bbf{H}}_k\bbf{P}_{k+1})=\tr(\bbf{\Lambda}_k)\geq\tr(\bbf{P}_k^T\widetilde{\bbf{H}}_k\bbf{P}_k),\]
        which yields \eqref{eq:mono-assume} with $\bbf{P}=\bbf{P}_k$, $\widehat{\bbf{P}}=\bbf{P}_{k+1}$, and $\eta=0$. Hence, by assumption~\ref{assump:mono}, we conclude $f(\bbf{P}_{k+1})\geq f(\bbf{P}_k)$.
    
        According to~\eqref{eq:kyfan_min}, we know that $\tr(\bbf{P}^T\bbf{C}_w(\bbf{P})\bbf{P})$ is not smaller than the sum of the $p$ smallest eigenvalues of $\bbf{C}_w(\bbf{P})$. Since the sum of the $p$ smallest eigenvalues of $\bbf{C}_w(\bbf{P})$ is assumed to be uniformly bounded from below for any orthonormal $\bbf{P}$, $\{f(\bbf{P}_k)\}_{k=0}^{\infty}$ is monotonically increasing and bounded and hence convergent.
    
        \item Since $\{\bbf{P}_k\}_{k=0}^{\infty}$ is a bounded sequence, it has a convergent subsequence $\{\bbf{P}_k\}_{k\in\mathbb{I}}$, converging to an orthonormal $\bbf{P}_*$. As $f(\bbf{P})$ is continuous with respect to $\bbf{P}$, by item (a), we have
        \[\lim_{k\to\infty} f(\bbf{P}_k)=\lim_{\mathbb{I}\ni k\to\infty} f(\bbf{P}_k)=f(\lim_{\mathbb{I}\ni k\to\infty} \bbf{P}_k)=f(\bbf{P}_*).\]
    
        \item Let $\omega$ be the uniform upper bound of the sum of the $p$ largest eigenvalues of $\bbf{C}_w(\bbf{P})$.
    
        Recall that we have~\eqref{eq:SCFcvgd}. $\{\bbf{P}_{k+1}\}_{k\in\mathbb{I}}$, as a bounded sequence, has a convergent subsequence $\{\bbf{P}_{k+1}\}_{k\in\mathbb{I}_1}$, converging to an orthonormal $\widehat{\bbf{P}}_*$, where $\mathbb{I}_1\subset\mathbb{I}$. Letting $\mathbb{I}_1\ni k\to\infty$ in \eqref{eq:SCFcvgd} yields
        \[\underbrace{\left[\bbf{C}_b(\bbf{P}_*)-\frac {\tr(\widehat{\bbf{P}}_*^T\bbf{C}_b(\bbf{P}_*)\widehat{\bbf{P}}_*)}{\tr(\widehat{\bbf{P}}_*^T\bbf{C}_w(\bbf{P}_*)\widehat{\bbf{P}}_*)}\,\bbf{C}_w(\bbf{P}_*)\right]}_{=:\widetilde{\bbf{H}}_*}\widehat{\bbf{P}}_*=\widehat{\bbf{P}}_*{\Lambda}_*,\]
        where the eigenvalues of $\bbf{\Lambda}_*$ consist of the $p$ largest eigenvalues of $\widetilde{\bbf{H}}_*$. 
        We claim that $\bbf{P}_*$ is an orthonormal eigenbasis matrix of $\widetilde{\bbf{H}}_*$ associated with its $p$ largest eigenvalues as well, just like $\widehat{\bbf{P}}_*$ is.
        Otherwise, according to~\eqref{eq:kyfan_max}, 
        \[\eta=\tr(\widehat{\bbf{P}}_*^T\widetilde{\bbf{H}}_*\widehat{\bbf{P}}_*)-\tr(\bbf{P}_*^T\widetilde{\bbf{H}}_*\bbf{P}_*)>0,\]
        yielding
        \[\frac {\tr(\widehat{\bbf{P}}_*^T\bbf{C}_b(\bbf{P}_*)\widehat{\bbf{P}}_*)}{\tr(\widehat{\bbf{P}}_*^T\bbf{C}_w(\bbf{P}_*)\widehat{\bbf{P}}_*)}=f(\bbf{P}_*)+\frac {\eta}{\tr(\bbf{P}_*^T\bbf{C}_w(\bbf{P}_*)\bbf{P}_*)}\geq f(\bbf{P}_*)+\frac {\eta}{\omega}.\]
        By continuity, there is an $k_0\in\mathbb{I}_1$ such that
        \[\frac {\tr(\bbf{P}_{k_0+1}^T\bbf{C}_b(\bbf{P}_{k_0})\bbf{P}_{k_0+1})}{\tr(\bbf{P}_{k_0+1}^T\bbf{C}_w(\bbf{P}_{k_0})\bbf{P}_{k_0+1})}\geq \frac {\tr(\widehat{\bbf{P}}_*^T\bbf{C}_b(\bbf{P}_*)\widehat{\bbf{P}}_*)}{\tr(\widehat{\bbf{P}}_*^T\bbf{C}_w(\bbf{P}_*)\widehat{\bbf{P}}_*)}-\frac 13\frac {\eta}{\omega},\,\,\]
        and
        \[f(\bbf{P}_{k_0})\geq f(\bbf{P}_*)-\frac 13 \frac {c\eta}{\omega}.\]
        Therefore,
        \[\frac {\tr(\bbf{P}_{k_0+1}^T\bbf{C}_b(\bbf{P}_{k_0})\bbf{P}_{k_0+1})}{\tr(\bbf{P}_{k_0+1}^T\bbf{C}_w(\bbf{P}_{k_0})\bbf{P}_{k_0+1})}\geq f(\bbf{P}_*)+\frac {\eta}{\omega}-\frac 13 \frac {\eta}{\omega}=f(\bbf{P}_*)+\frac 23 \frac {\eta}{\omega}\geq f(\bbf{P}_{k_0})+\frac 23 \frac {\eta}{\omega}.\]
        By assumption~\ref{assump:mono}, we get
        \[f(\bbf{P}_{k_0+1})\geq f(\bbf{P}_{k_0})+c\frac 23\frac {\eta}{\omega}\geq f(\bbf{P}_*)-\frac 13 \frac {c\eta}{\omega}+c\frac 23\frac {\eta}{\omega}=f(\bbf{P}_*)+\frac 13 \frac {c\eta}{\omega}>f(\bbf{P}_*),\]
        contradicting all $f(\bbf{P}_k)\le f(\bbf{P}_*)$.
    
        The contradiction indicates that $\bbf{P}_*$ is an orthonormal eigenbasis matrix of $\widetilde{\bbf{H}}_*$ associated with its $p$ largest eigenvalues, and hence $\eta=0$, i.e.,
        \[\frac {\tr(\widehat{\bbf{P}}_*^T\bbf{C}_b(\bbf{P}_*)\widehat{\bbf{P}}_*)}{\tr(\widehat{\bbf{P}}_*^T\bbf{C}_w(\bbf{P}_*)\widehat{\bbf{P}}_*)}=f(\bbf{P}_*),\]
        implying $\widetilde{\bbf{H}}_*=\bbf{H}_*$. Therefore, in conclusion, $\bbf{P}_*$ is an orthonormal eigenbasis matrix of $\bbf{H}_*$ associated with its $p$ largest eigenvalues.
    
        \item In proving item (c), we concluded that $\widetilde{\bbf{H}}_*=\bbf{H}_*$ and so we have $\bbf{H}_*\widehat{\bbf{P}}_*=\widehat{\bbf{P}}_*\bbf{\Lambda}_*$ upon letting $\mathbb{I}_1\ni k\to\infty$ in \eqref{eq:SCFcvgd}, where the eigenvalues of $\bbf{\Lambda}_*$ consist of the $p$ largest eigenvalue of $\bbf{H}_*$.
        By item (c), $\bbf{P}_*$ and $\widehat{\bbf{P}}_*$ are two orthonormal eigenbasis matrices of $\bbf{H}_*$ associated with its $p$ largest eigenvalues. By the assumption that $\bbf{\lambda}_p(\bbf{H}_*)-\bbf{\lambda}_{p+1}(\bbf{H}_*)>0$, we conclude $\widehat{\bbf{P}}_*=\bbf{P}_*\bbf{Q}$ where $\bbf{Q}\in\RR^{p\times p}$ is an orthogonal matrix. Hence, we have
        \[\bbf{P}_*=\arg\max_{\bbf{P}^T\bbf{P}=I_p}\frac {\tr(\bbf{P}^T\bbf{C}_b(\bbf{P}_*)\bbf{P})}{\tr(\bbf{P}^T\bbf{C}_w(\bbf{P}_*)\bbf{P})}.\]
    \end{enumerate}
\end{proof}


\section{Numerical Experiments}\label{sec:num}

In this section, we demonstrate the convergence, classification accuracy, and scalability of WDA-nepv. The outline of the section is as follows: first, we provide a brief summary of datasets that we perform numerical experiments on. Then, we demonstrate the convergence behaviors of WDA-nepv, followed by a discussion on non-optimal convergence.
Using the K-Nearest-Neighbors (KNN) algorithm to measure the classification accuracy, we demonstrate that WDA-nepv performs either competitively or better than WDA-gd and WDA-eig. Finally, we demonstrate that WDA-nepv scales linearly in subspace dimension $p$, and quadratically in dimension $d$ and in number of data points $n$.

The experiments were conducted using Python on a PC with an Intel Core i7-7500U processor with 16GB of RAM.
To advocate for reproducible research, we share our Python implementation of WDA-nepv for the experiment results presented in this paper\footnote{Github page for WDA-nepv: \url{https://github.com/gnodking7/WDAnepv}}. The codes of WDA-gd and WDA-eig are provided by their respective authors\footnote{Code for WDA-gd: \url{https://pythonot.github.io/auto_examples/others/plot_WDA.html}. \newline
Github page for WDA-eig: \url{https://github.com/HexuanLiu/WDA_eig}}.
For all experiments, the initial projection $\bbf{P}_0$ is chosen as a random orthogonal matrix. Unless otherwise stated, the stopping tolerance parameter is preset at $10^{-5}$. 

\subsection{Datasets}\label{subsec:data}

\paragraph{Synthetic dataset.}
We consider the synthetic dataset used in \cite{flamary2018wasserstein,liu2020ratio}.
The dataset consists of three bi-modal classes such that the corresponding three data matrices are
\begin{equation}\label{eq:syn_data}
    \bbf{X}^1\in\RR^{d\times n_1},\quad \bbf{X}^2\in\RR^{d\times n_2},\quad \bbf{X}^3\in\RR^{d\times n_3}.\quad 
\end{equation}
Each data point is a $d$ dimensional vector whose first two components are discriminative and the remaining components are Gaussian noise, i.e., drawn from the standard normal distribution $\mathcal{N}(0,1)$. In particular, each class consists of two separate modes in its discriminative components such that the number of data points is equally split among the two modes.

The discriminative behavior is shown in Figure~\ref{fig:WDA_Synth} for an example $n_1=30, n_2=40, n_3=30$: the left subplot shows the first two components and the right subplot the next two components.

\begin{figure}[H]
\centering
\includegraphics[width=1.0\textwidth]{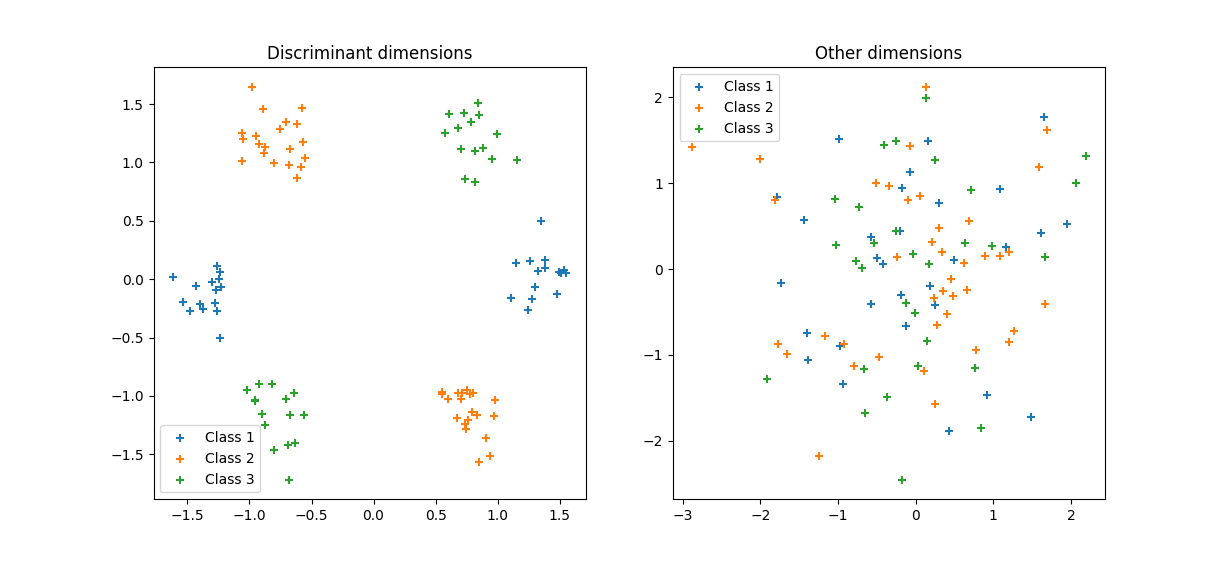}
\caption{Each of the three classes is displayed in a different color. The left figure is the plot of the first two components of the data points, and the right figure is the plot of the next two components.} \label{fig:WDA_Synth}
\end{figure}

\paragraph{Synthetic Shape dataset.}
We consider various synthetic 2 dimensional datasets from \cite{ClusteringDatasets} that follow specific patterns and shapes: Jain \cite{jain2005data}, Flame \cite{fu2007flame}, Pathbased \cite{chang2008robust}, Compound \cite{zahn1971graph}, Aggregation \cite{gionis2007clustering}, R15 \cite{veenman2002maximum}. In order to test an algorithm's ability to uncover the discriminating features we append these datasets with 8 additional components drawn from the standard normal distribution $\bbf{N}(0,1)$. Their first two components, i.e., the discriminating features, are plotted in Figure~\ref{fig:WDA_Synth_shapes}.

\begin{figure}[h]
\centering
\includegraphics[width=1\textwidth]{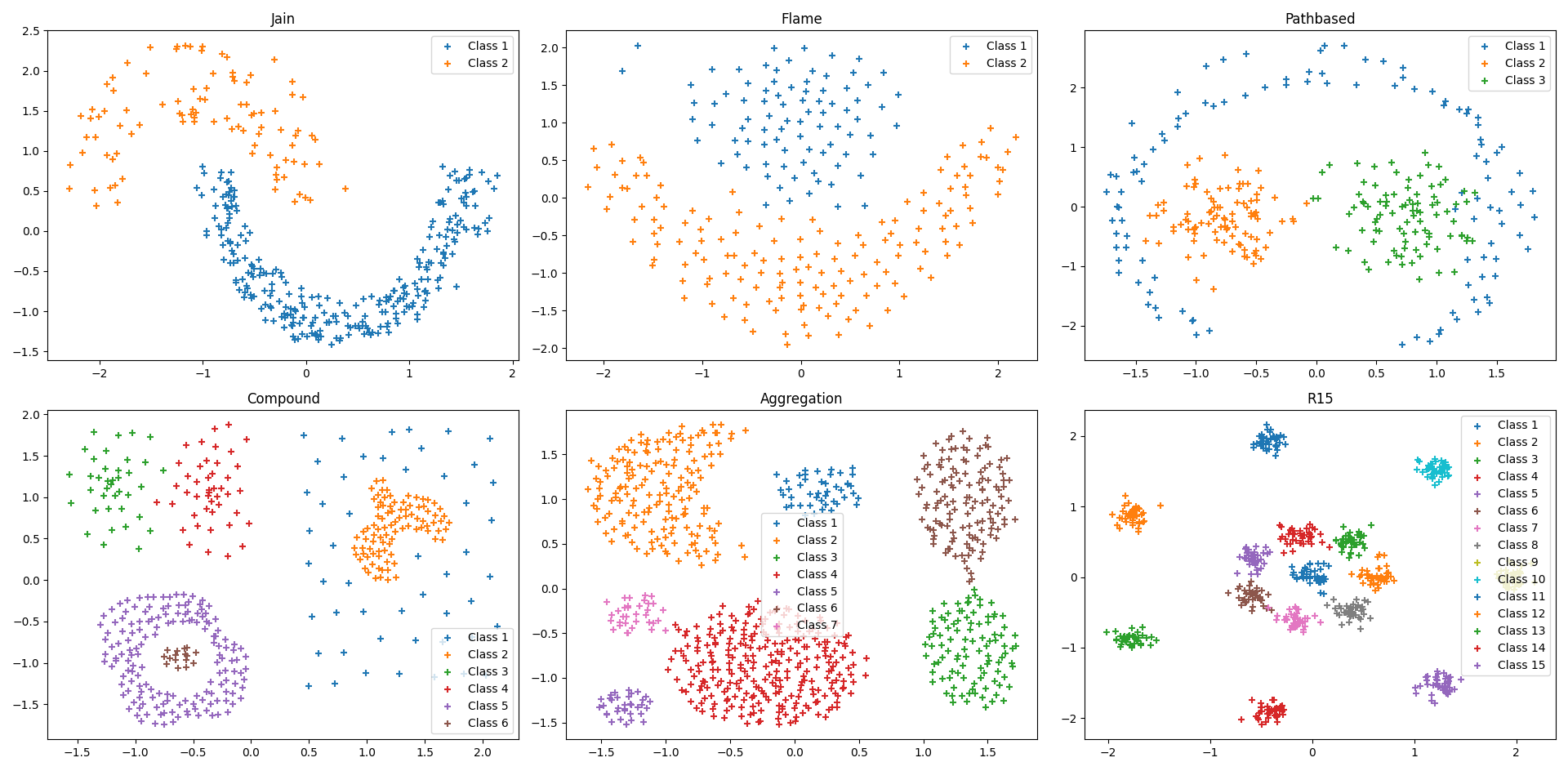}
\caption{Plots of the first two components of shape sets that are discriminatory.} \label{fig:WDA_Synth_shapes}
\end{figure}

\paragraph{UCI Datasets.}
The UCI Repository \cite{Dua:2019} is a collection of datasets that are widely used by the machine learning community for the empirical analysis of machine learning algorithms. In particular, the repository provides a wide variety of datasets that are suitable for clustering and classification tasks. Among the datasets, we choose the real-life datasets named Iris, Wine, Ionosphere, LSVT, and Parkinson's Disease. In Table~\ref{tab:UCI}, their dimension size, number of data points, and number of classes are displayed.


\begin{table}[H]
\centering 
\begin{tabular}{c|ccccc}
 & \textbf{Iris} & \textbf{Wine} & \textbf{Ionosphere} & \textbf{LSVT} & \textbf{Parkinson} \\ \hline
\textbf{Dimension ($d$)} & 4 & 13 & 34 & 309 & 754 \\
\textbf{Data points ($\sum n_c$)} & 150 & 178 & 351 & 126 & 756 \\
\textbf{Classes ($C$)} & 3 & 3 & 2 & 2 & 2
\end{tabular}
\caption{UCI datasets description.}
\label{tab:UCI}
\end{table}

\subsection{WDA-nepv: Convergence Behavior}\label{subsec:conv}

\begin{example} \label{eg:conv1} 
{\rm 
To investigate the convergence behavior of WDA-nepv, we use a synthetic dataset with a dimension of $d=10$ and three classes of data points of sizes $(n_1,n_2,n_3)=(30,40,30)$. We perform two experiments to investigate the convergence behavior of WDA-nepv: 
\begin{itemize}
    \item With fixed regularization parameter $\lambda=0.01$, the convergence behavior of WDA-nepv is depicted for various subspace dimensions $p\in\{1,2,3,4,5\}$; see Figure~\ref{fig:WDA_Synth_Conv}(a).
    \item With fixed subspace dimension $p=2$, the convergence behavior of WDA-nepv is reported for various regularization parameters $\lambda\in\{0.001,0.01,0.1\}$; see Figure~\ref{fig:WDA_Synth_Conv}(b).
\end{itemize}

\begin{figure}[H]
\centering
\subfloat[Varying $p$]{{\includegraphics[width=0.45\textwidth]{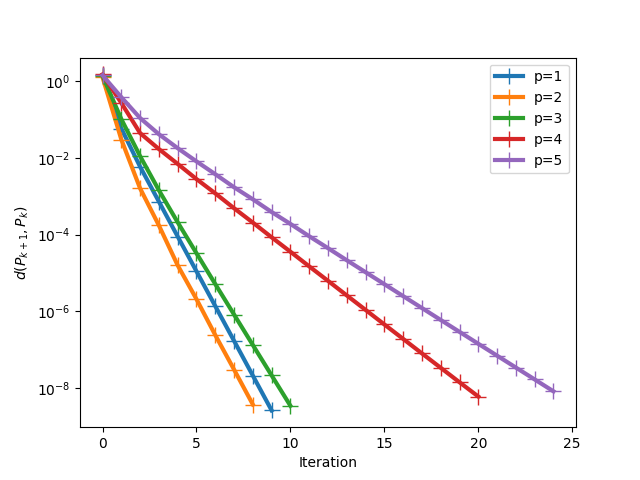} }}%
\quad
\subfloat[Varying $\lambda$]{{\includegraphics[width=0.45\textwidth]{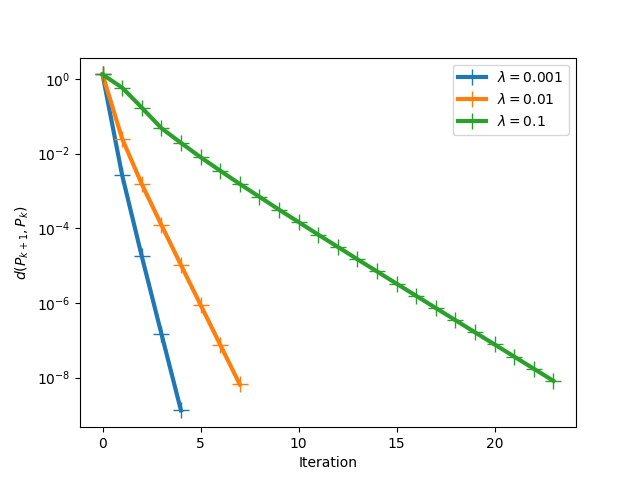} }}%
\caption{Convergence behavior of WDA-nepv on the synthetic dataset.}
\label{fig:WDA_Synth_Conv}
\end{figure}

Both plots in Figure~\ref{fig:WDA_Synth_Conv} illustrate linear convergence for WDA-nepv regardless of the choice of the regularization parameter or the subspace dimension. 
WDA-nepv achieves the fastest convergence rate for the subspace dimension $p=2$, the dimension size of the true discriminative subspace of the dataset. 
For a similar subspace dimension size $p=1$ WDA-nepv achieves a similar convergence rate as $p=2$, while as the subspace dimension $p$ increases the convergence rate slows down. 
We also observe that the convergence rate slows down as the regularization parameter $\lambda$ increases. When $\lambda$ is small, NTRopt-WDA~\eqref{eq:WDA_tropt} can be considered as LDA with a small perturbation to the cross-covariance matrices. In this case, as LDA is known to have a local quadratic convergence \cite{cai2018eigenvector}, WDA-nepv converges faster than the cases where $\lambda$ is larger.

}
\end{example}

\subsubsection{Many Local Maxima and Non-optimal Convergence}

\paragraph{Many local maxima.}
NTRopt-WDA~\eqref{eq:WDA_tropt} is a highly nonlinear, non-convex bi-level optimization that often has many local optimizers. The following example illustrates this.

\begin{example}\label{ex:local}
{\rm
Let us consider the Synthetic dataset with $d=2$ and the number of data points $n_1=30, n_2=40, n_3=30$. 
We desire to compute a projection $\bbf{P}$  that projects these data points onto a subspace of dimension $p=1$, i.e., $\bbf{P}$ is a normalized vector $\bbf{p}\in\RR^{2 \times 1}$.
We define $f(\bbf{p})$ as the value of the objective function of NTRopt-WDA at a normalized vector $\bbf{p}$. That is,
\begin{eqnarray}\label{eq:obj_vec}
    f(\bbf{p}):=\frac{\tr(\bbf{p}^T\bbf{C}_b(\bbf{p})\bbf{p})}{\tr(\bbf{p}^T\bbf{C}_w(\bbf{p})\bbf{p})}.
\end{eqnarray}
Figure~\ref{fig:WDA_Synth_3D} plots the values $f(\bbf{p})$ (gray points) at $2000$ random normalized vectors $\bbf{p}$, where the regularization parameter is chosen as $\lambda=0.1$ for plot (a) and $\lambda=1$ for plot (b). 
We observe that for both regularization parameters, there are many local maxima.

\begin{figure}[h]
\centering
\subfloat[$\lambda=0.1$]
{{\includegraphics[width=0.5\textwidth]{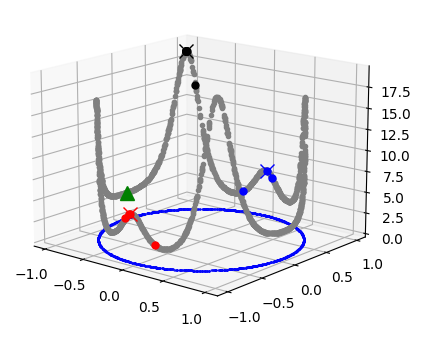} }}%
\subfloat[$\lambda=1$]
{{\includegraphics[width=0.5\textwidth]{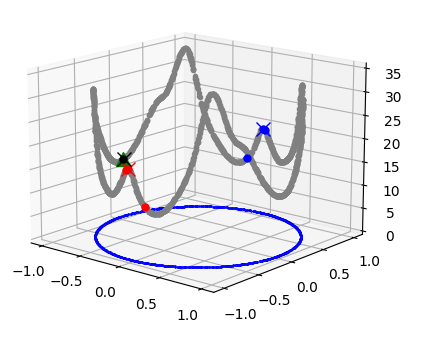} }}%
\caption{Plots of NTRopt-WDA values $f(\bbf{p}_k)$ (gray points) at random normalized vectors, with regularization parameter $\lambda=0.1$ in (a) and $\lambda=1$ in (b). For all three considered algorithms, the initial projection $\bbf{p}_0$ (a green triangle) is chosen as a local minimizer. The values $f(\bbf{p}_k)$ of WDA-nepv (red circles), WDA-gd (black circles), and WDA-eig (blue circles) are displayed for the projections $\bbf{p}_k$ obtained at each iteration of an algorithm. For each algorithm, the value $f(\bbf{p}_*)$ for the converged projection $\bbf{p}_*$ is displayed as the `$\times$' symbol.}
\label{fig:WDA_Synth_3D}
\end{figure}
}
\end{example}

\paragraph{Non-optimal convergence.}

Finding a global optimizer of a non-convex problem is NP-hard. Algorithms for non-convex problems are prone to converge towards a local optimizer, and consequently, lead to a suboptimal solution. A suboptimal solution is not desirable in optimization and for the task at hand. For instance, in a classification task, relying on a suboptimal projection often proves insufficient for identifying the subspace that effectively discriminates the projected data points.

\begin{example}\label{ex:conv2}
{\rm
In order to analyze the convergence behaviors of WDA-gd, WDA-eig and WDA-nepv, we examine the history of their NTRopt-WDA values for the problem described in Example~\ref{ex:local}. That is, for projections $\bbf{p}_k$ obtained at each iteration of an algorithm, we examine how the values $f(\bbf{p}_k)$~\eqref{eq:obj_vec} are changing. Figure~\ref{fig:WDA_Synth_3D} illustrates the local convergence behavior for WDA-nepv, WDA-gd, and WDA-eig.

For both regularization parameters $\lambda\in\{0.1,1\}$, the same initial projection $\bbf{p}_0$ corresponding to a local minimizer is used and shared by all three algorithms. In the figure, the value $f(\bbf{p}_0)$ is represented as a green triangle, while the values $f(\bbf{p}_k)$ are shown as red circles for WDA-nepv, black circles for WDA-gd, and blue circles for WDA-eig. The final value $f(\bbf{p}_*)$ for the converged projection $\bbf{p}_*$ is displayed as the `$\times$' symbol, with each color corresponding to the respective algorithm.

In Figure~\ref{fig:WDA_Synth_3D}(a),
we observe that for a regularization parameter $\lambda=0.1$, WDA-nepv and WDA-eig converge towards a local maximizer while WDA-gd successfully finds a global maximizer. 
However, in Figure~\ref{fig:WDA_Synth_3D}(b), when $\lambda=1$,
we also observe that WDA-gd fails to find a global maximizer and remains trapped in the initial local minimizer.
}
\end{example}

\begin{example}\label{ex:conv3}
{\rm
In addition, we use the UCI datasets Wine and Iris to further illustrate the convergence towards non-global local optimizers.
We set the regularization parameter as $\lambda=0.01$ for the Wine dataset and $\lambda=1$ for the Iris dataset.
For both datasets, we solve the NTRopt-WDA problem for subspace dimensions $p\in\{2,3\}$. The history of the NTRopt-WDA values, denoted as $f(\bbf{P}_k)$, is plotted in Figure~\ref{fig:WDA_WINE_IRIS_Obj}. All three algorithms share the same initial projection $\bbf{P}_0$, resulting in the same value of $f(\bbf{P}_0)$.

\begin{figure}[h]
    \centering
    \includegraphics[width=\linewidth]{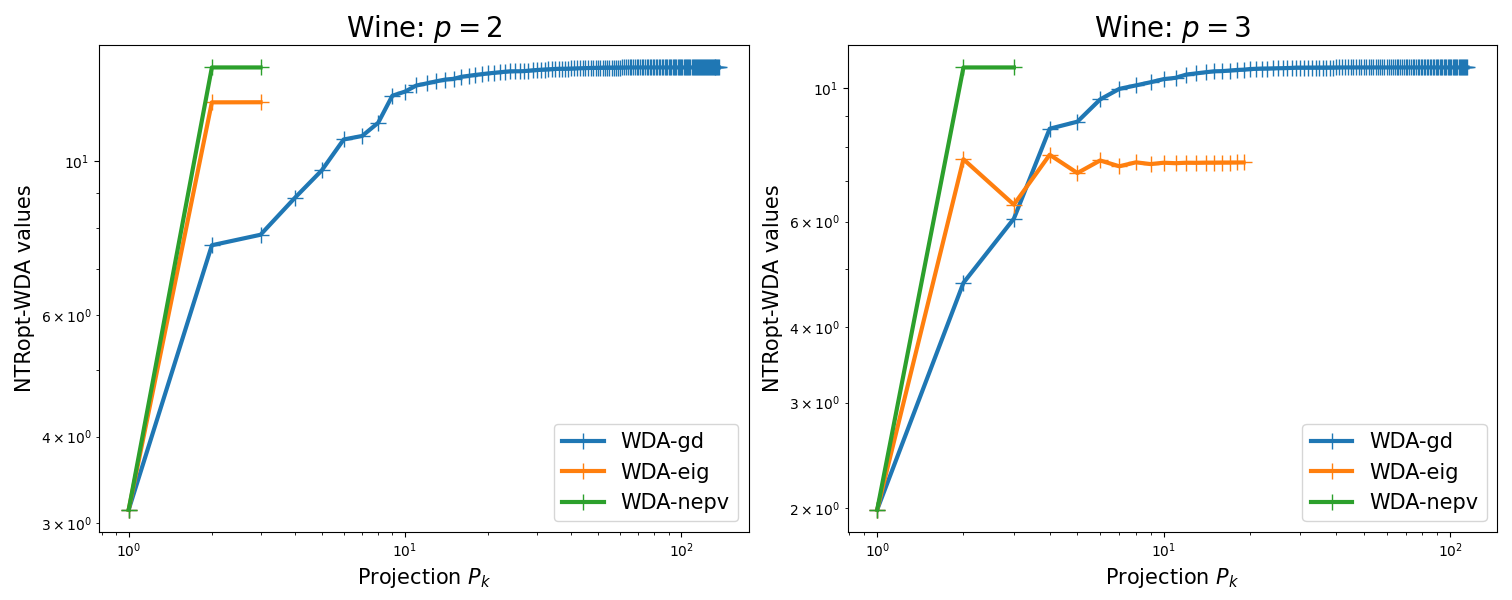}
    \includegraphics[width=\linewidth]{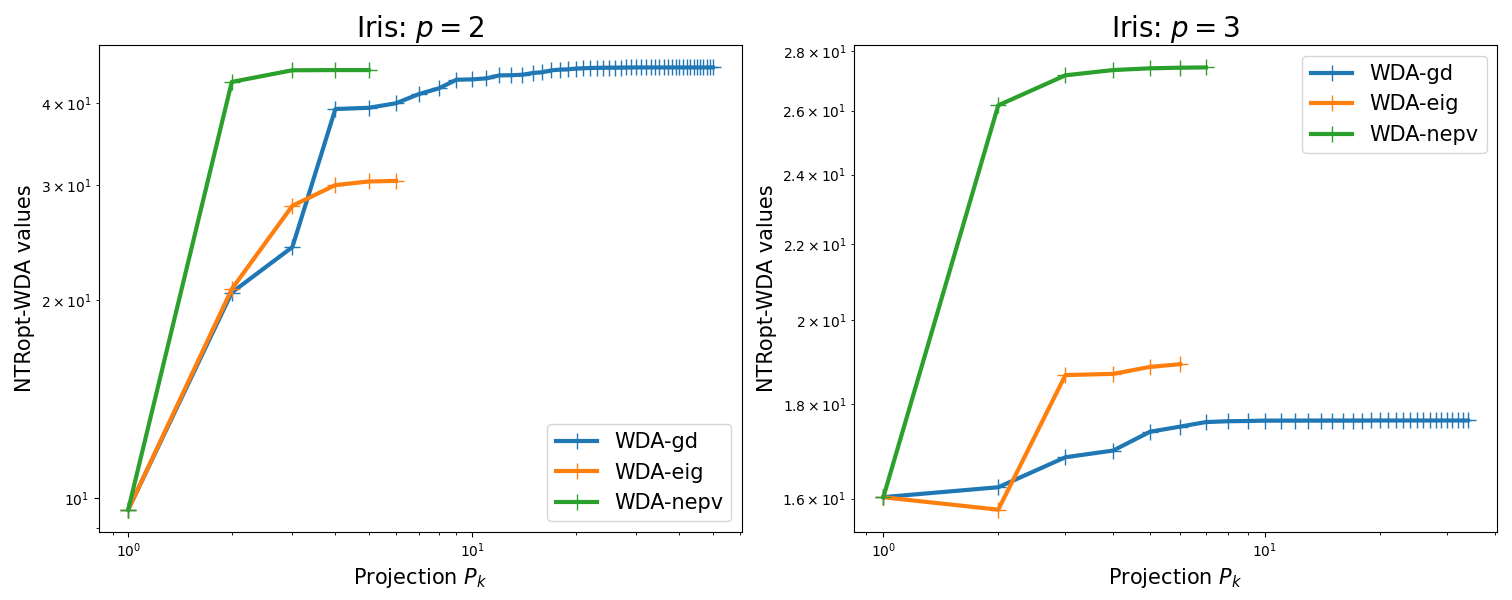}
    \caption{History of the NTRopt-WDA values for Wine and Iris. Subspace dimensions $p\in\{1,2,3\}$ are considered and regularization parameters are set at $\lambda=0.01$ and $\lambda=1$ for Wine and Iris, respectively.}
    \label{fig:WDA_WINE_IRIS_Obj}
\end{figure}

For the Wine dataset, we observe that both WDA-gd and WDA-nepv achieve the same optimal value, $f(\bbf{P}_*)$, while WDA-eig's optimal value, $f(\bbf{P}_*)$, is smaller for both $p=2$ and $p=3$, indicating that WDA-eig found a suboptimal solution. 
This highlights the suboptimality of WDA-eig's surrogate ratio trace model~\eqref{eq:WDA_rtopt}.

Regarding the Iris dataset, WDA-nepv obtains the largest optimal value, $f(\bbf{P}_*)$, among the three algorithms for all considered subspace dimensions.
Moreover, WDA-nepv demonstrates faster convergence, requiring significantly fewer iterations compared to WDA-gd.
}
\end{example}

Both Example~\ref{ex:conv2} and Example~\ref{ex:conv3} illustrate that convergence to non-global optima can be a prevalent issue for all three algorithms.
The convergence behavior varies across datasets, with instances where some or all three algorithms successfully converge to the same optimizer, as well as instances where they converge to different optimizers.
Notably, when the initial projection is chosen close to the global maximizer, it can lead to global convergence. 
However, since the global maximizer is typically unknown in advance, obtaining an initial projection close to it may require multiple attempts.

\paragraph{Different choices of initial projection.}~
While all the results in the paper utilize a random initial projection $P_0$, we acknowledge that alternative choices, such as initializing with the solution of PCA, are possible.
Notably, for low-dimensional datasets like Wine ($d=13$) and Iris ($d=4$), we observe no performance difference between random and PCA initialization.
However, for higher-dimensional datasets like LSVT ($d=309$) and Parkinson ($d=754$), the optimal objective value obtained from a PCA initialization is higher than that from a random initialization.

It is worth noting that despite the PCA initialization yielding a higher objective value, the random initialization achieves greater classification accuracy for these high-dimensional datasets.
This finding underscores the complexity and multi-faceted nature of the optimization problem, where an improved objective value does not always directly translate to better classification performance.

\subsection{WDA-nepv: Classification Accuracy}

One of the many goals of DR methods is to derive a projection such that the projected data vectors maintain or amplify the coherent structure of the original dataset. For a supervised linear DR method, one hopes to obtain an optimal projection matrix $\bbf{P}$ such that the class structure of the original data vectors in the projected subspace is more pronounced. For this reason, the effectiveness of a supervised linear DR method is often measured by the accuracy of classification on the projected data vectors. Measuring the accuracy of NTRopt-WDA~\eqref{eq:WDA_tropt} is no exception. To evaluate the classification accuracy of an algorithm, we follow the following conventional steps:
\begin{enumerate}
    \item Randomly divide a given dataset into a training dataset and a testing dataset, with equal size of $50\%$ each.
    \item Compute the optimal projection matrix $\bbf{P}_*$ using the training dataset.
    \item Project the testing dataset onto the lower-dimensional subspace using the optimal projection matrix $\bbf{P}_*$.
    \item Employ the K-Nearest-Neighbors classifier (KNN) on the projected testing dataset to compute the classification accuracy.
\end{enumerate}
The classification accuracy is quantified in terms of prediction error. A smaller error indicates better performance. 

\begin{example}
{\rm
In this example, the synthetic shape datasets are used to evaluate the classification accuracy of WDA-nepv, WDA-gd, and WDA-eig. The projection dimension is fixed as $p=2$, the dimension size of the true discriminative subspace of the datasets. We consider KNN number $K=10$ with various regularization parameters $\lambda=\{0.1,1,5\}$. The experiment is repeated 100 times and the average classification errors is reported in Table~\ref{tab:WDA_Synth_Shape}.

\begin{table}[H]
\centering
\begin{tabular}{llllllll}
$\lambda$     & Alg               & Jain           & Flame          & Pathbased      & Compound       & Aggregation    & R15            \\ \hline
$\lambda=0.1$ & WDA-gd            & \textbf{0.042}          & 0.101          & \textbf{0.106}          & \textbf{0.089}          & \textbf{0.003}          & \textbf{0.005}          \\
              & WDA-eig           & 0.062          & \textbf{0.050}          & 0.126          & 0.093          & \textbf{0.003}          & \textbf{0.005}          \\
              & WDA-nepv & \textbf{0.042} & 0.128 & 0.148 & 0.092 & \textbf{0.003} & \textbf{0.005} \\ \hline
$\lambda=1$   & WDA-gd            & 0.059          & 0.112          & 0.092          & 0.079          & \textbf{0.003}          & \textbf{0.004}          \\
              & WDA-eig           & 0.061          & \textbf{0.076}          & \textbf{0.073}          & 0.080          & \textbf{0.003}          & \textbf{0.004}          \\
              & WDA-nepv & \textbf{0.021} & 0.081 & 0.079 & \textbf{0.078} & \textbf{0.003} & \textbf{0.004} \\ \hline
$\lambda=5$   & WDA-gd            & 0.068          & 0.142          & 0.220          & 0.232          & 0.006          & 0.009          \\
              & WDA-eig           & 0.053          & \textbf{0.088}          & \textbf{0.101}          & \textbf{0.073}          & \textbf{0.003}          & \textbf{0.004}          \\
              & WDA-nepv & \textbf{0.046} & 0.118 & 0.159 & 0.074 & \textbf{0.003} & \textbf{0.004}
\end{tabular}
\caption{Prediction errors of WDA-gd, WDA-eig, WDA-nepv for Jain, Flame, Pathbased, Compound, Aggregation, R15 dataset.}
\label{tab:WDA_Synth_Shape}
\end{table}

For Aggregation and R15, WDA-nepv achieves similar high accuracy as WDA-gd and WDA-eig. For Compound, WDA-nepv has smallest prediction error for $\lambda=1$ and while it does not have the smallest prediction errors for $\lambda=0.1$ and $\lambda=5$, less than $0.3\%$ accuracy difference exists between WDA-nepv and the best performing algorithm. For Pathbased, WDA-nepv does not have the smallest prediction errors. However, when $\lambda=1$, for which all three algorithms have the smallest prediction errors among the three regularization parameters indicating that $\lambda=1$ is optimal for Pathbased, WDA-nepv has only $0.6\%$ difference with the best performing algorithm. For Flame, WDA-nepv does not have the smallest prediction errors but it performs better than WDA-gd for $\lambda=1$ and $\lambda=5$. For Jain dataset, WDA-nepv has the smallest prediction errors for all regularization parameters, achieving around $2\%$ to $4\%$ higher accuracy.
}
\end{example}

\begin{example}{\rm 
Using the UCI datasets Wine, Ionosphere, LSVT, and Parkinson's Disease, the classification accuracy of WDA-nepv is compared with the classification accuracy of WDA-gd and WDA-eig. As suggested in WDA-eig \cite{liu2020ratio}, we add a small perturbation term $\epsilon I_d$ with $\epsilon = 1$ to the matrix $\bbf{C}_w(\bbf{P}_k)$. This enforces the positive definiteness of $\bbf{C}_w(\bbf{P}_k)$ in practice and helps the robustness of the computation of the eigenvalue problems.
The regularization parameter $\lambda=0.01$ is fixed throughout and the stopping tolerance parameter is set at $10^{-5}$. Each experiment is repeated 20 times, and the average and the min-max interval of classification errors are reported.
Two different experiments are considered; one to observe the algorithm behavior in KNN number and another in subspace dimension.
\begin{itemize}
    \item With fixed subspace dimension $p=5$, various KNN numbers $K\in\{1,3,5,7,9,11,13,15,17,19\}$ are considered; see left column of Figure~\ref{fig:WDA_UCI_Class}.
    
    \item With fixed KNN number $K=11$, various subspace dimensions $p\in\{1,2,3,4,5\}$ are considered; see right column of Figure~\ref{fig:WDA_UCI_Class}.
\end{itemize}


\begin{figure}[]
\begin{center}
\includegraphics[width=0.8\textwidth]{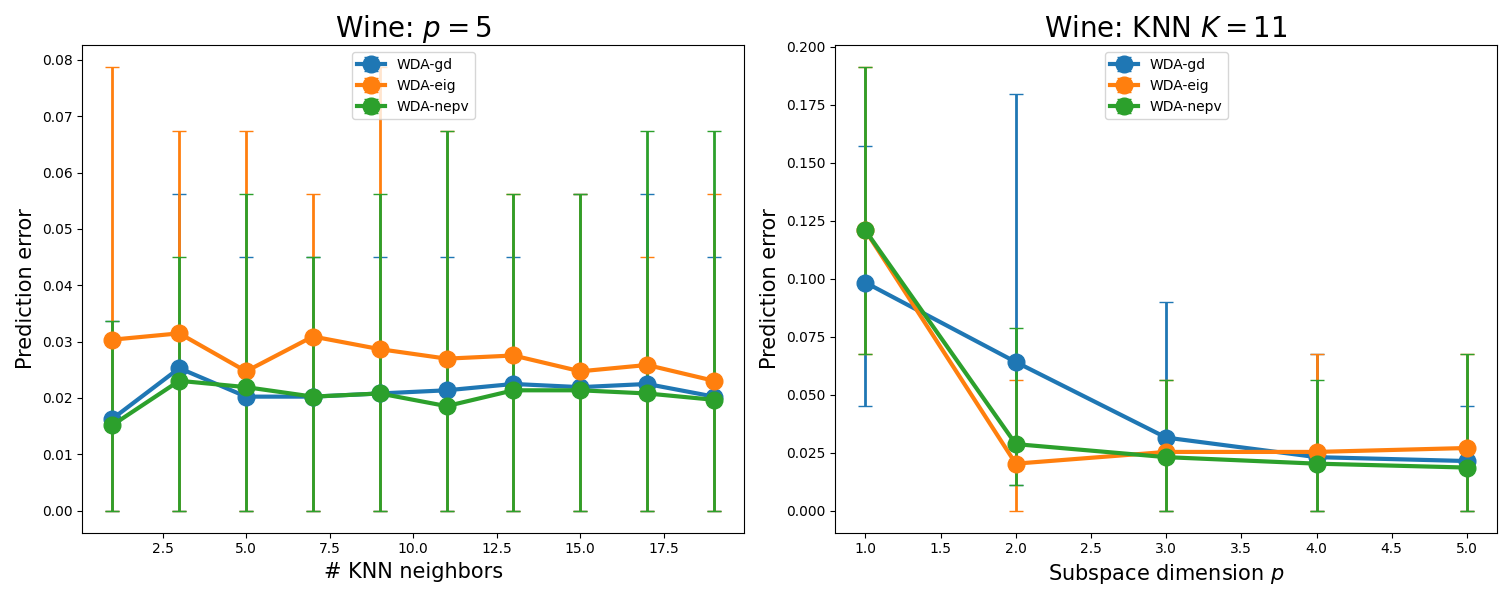}
\includegraphics[width=0.8\textwidth]{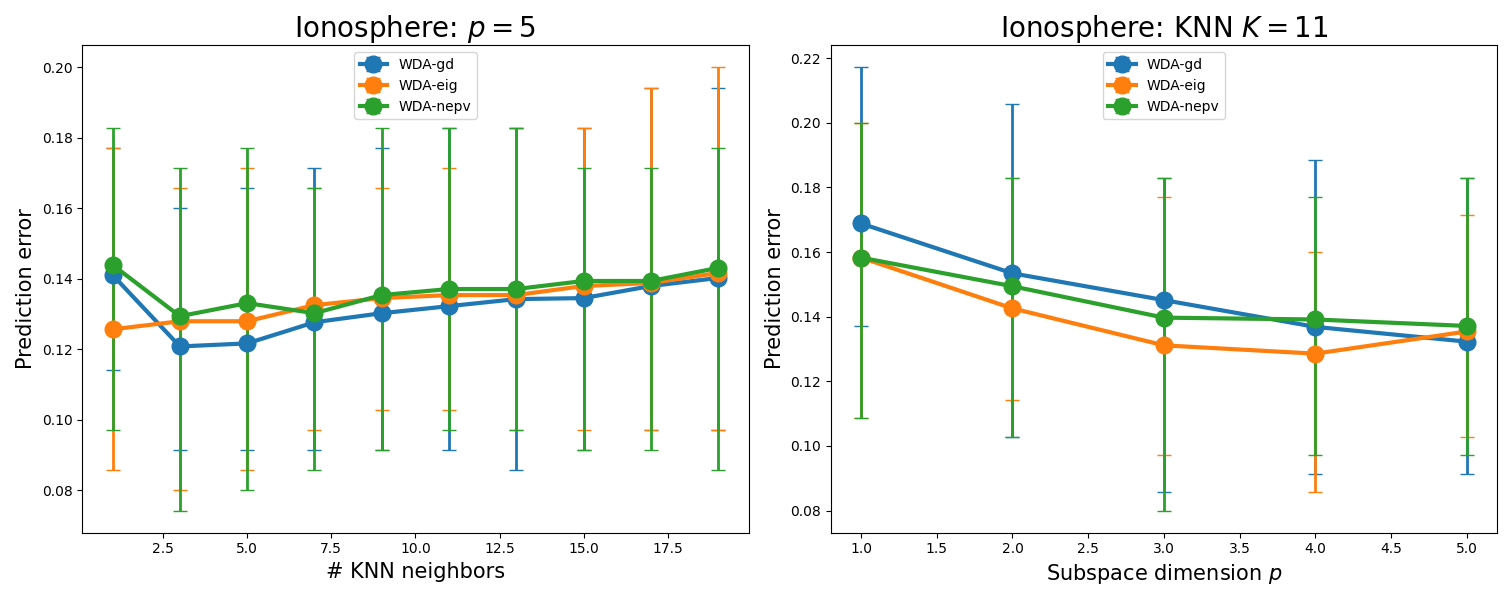}

\includegraphics[width=0.8\textwidth]{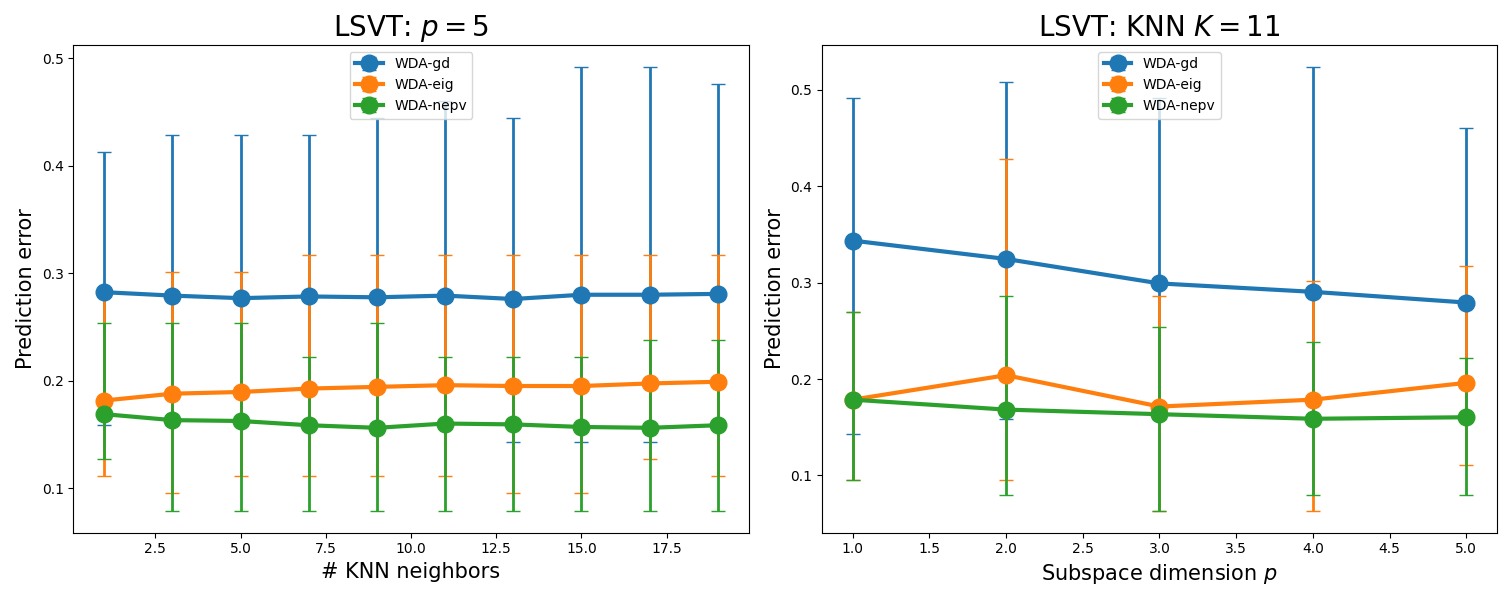}
\includegraphics[width=0.8\textwidth]{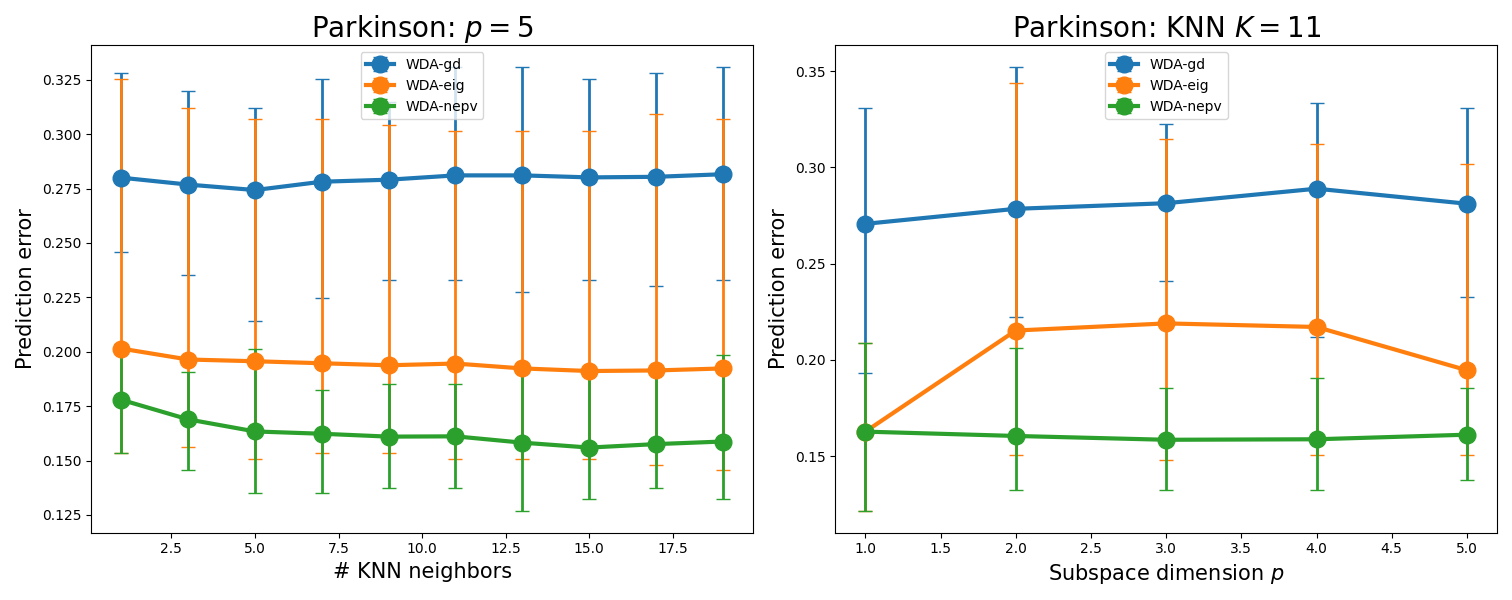}
\end{center}
\caption{Prediction errors for WDA-gd, WDA-eig, and WDA-nepv.}
\label{fig:WDA_UCI_Class}
\end{figure}


For Wine and Ionosphere, we observe that WDA-nepv performs comparable to WDA-gd and WDA-eig, achieving similar accuracy as the other two WDA algorithms - when looking at the averages, less than 1\% classification accuracy difference exists between WDA-nepv and the best performing algorithm for all considered values of KNN numbers and subspace dimensions.

For higher dimensional datasets LSVT and Parkinson, WDA-nepv achieves lowest average prediction errors and has the smallest min-max error intervals among the three algorithms. For LSVT, WDA-nepv achieves, on average, at least 2\% to 4\% higher classification accuracy for considered KNN numbers, and at least 1\% to 4\% higher classification accuracy for considered subspace dimensions. For Parkinson, WDA-nepv achieves, on average, at least 2\% to 4\% higher classification accuracy for considered KNN numbers, and at least 3\% to 5\% higher classification accuracy for considered subspace dimensions.

}
\end{example}

\subsection{WDA-nepv: Timing and Scalability}

We demonstrate the efficiency of WDA-nepv by reporting its running time and its scalability. 

\begin{example}\label{eg:time}
{\rm
We compare the running time of WDA-gd, WDA-eig, and WDA-nepv on the UCI datasets Wine, Ionosphere, LSVT, and Parkinson. For the fixed regularization parameter $\lambda=0.01$, we measure the running time for different subspace dimensions $p\in\{3,4,5\}$ for Wine and Ionosphere, and $p\in\{15,20,25\}$ for larger dimensional datasets LSVT and Parkinson. We repeat the experiment 20 times and report the average running times are shown in the following table:

\begin{table}[H]
\centering
\subfloat[Wine]
{\resizebox{0.45\linewidth}{!}{\begin{tabular}{|l|l|l|l|}
\hline
         & $p=3$          & $p=4$          & $p=5$          \\ \hline
WDA-gd   & 1.854          & 2.354          & 2.078          \\ \hline
WDA-eig  & \textbf{0.019} & \textbf{0.018} & \textbf{0.017} \\ \hline
WDA-nepv & 0.031          & 0.028          & 0.028       \\ \hline  
\end{tabular}}}
\quad
\subfloat[Ionosphere]
{\resizebox{0.47\linewidth}{!}{\begin{tabular}{|l|l|l|l|}
\hline
         & $p=3$          & $p=4$          & $p=5$          \\ \hline
WDA-gd   & 12.234          & 12.530          & 11.896          \\ \hline
WDA-eig  & 0.868 & 0.902 & 0.891 \\ \hline
WDA-nepv & \textbf{0.242}          & \textbf{0.241}       & \textbf{0.284}         \\ \hline
\end{tabular}}}
\\
\subfloat[LSVT]
{\resizebox{0.45\linewidth}{!}{\begin{tabular}{|l|l|l|l|}
\hline
         & $p=15$         & $p=20$         & $p=25$         \\ \hline
WDA-gd   & 11.332         & 12.965         & 14.106         \\ \hline
WDA-eig  & 13.689         & 14.007         & 14.153         \\ \hline
WDA-nepv & \textbf{4.683} & \textbf{4.811} & \textbf{5.068} \\ \hline
\end{tabular}}}
\quad
\subfloat[Parkinson]
{\resizebox{0.48\linewidth}{!}{\begin{tabular}{|l|l|l|l|}
\hline
         & $p=15$         & $p=20$         & $p=25$         \\ \hline
WDA-gd   & 1060.76         & 1000.56         & 1001.29         \\ \hline
WDA-eig  & 258.33         & 285.68         & 281.66         \\ \hline
WDA-nepv & \textbf{180.32} & \textbf{189.82} & \textbf{205.89} \\ \hline
\end{tabular}}}
\label{tab:time}
\end{table}

For a small dimensional Wine, we observe that while WDA-eig has the shortest running time, the running time of WDA-nepv only differs from WDA-eig by around 0.01 seconds. For larger dimensional Ionosphere, LSVT, and Parkinon, however, WDA-nepv achieves a running time that is more than half of the running time of WDA-gd and WDA-eig.
Unlike WDA-gd, which incurs heavy computational costs from computing the derivatives, WDA-nepv is derivative-free and has a short running time. Additionally, the efficient use of level-3 BLAS for the cross-covariance matrices gives WDA-nepv an edge over WDA-eig in terms of running time.

}
\end{example}

\begin{example}{\rm 
We demonstrate that WDA-nepv scales linearly in subspace dimension $p$ 
and quadratically in data dimension $d$, and number of data points $n$ using the Synthetic dataset.
The regularization parameters $\lambda=0.1,0.01,0.001$ are considered, and the reported scalability results are averaged over 100 trials.
Three experiments are performed:
\begin{itemize}
    \item With dimension $d=10$ and number of total data points $n=100$ with $(n_1,n_2,n_3)=(30,40,30)$, various subspace dimensions $p\in\{1,2,3,4,5\}$ are considered.
    \item With subspace dimension $p=2$ and number of total data points $n=100$ with $(n_1,n_2,n_3)=(30,40,30)$, various dimensions $d\in\{80,160,320,640,1280,2560\}$ are considered.
    \item With dimension $d=50$ and subspace dimension $p=2$, various number of total data points $n\in\{100,200,300,500,1000\}$ are considered.
\end{itemize}

\begin{figure}[h]
    \centering
    \includegraphics[width=0.32\linewidth]{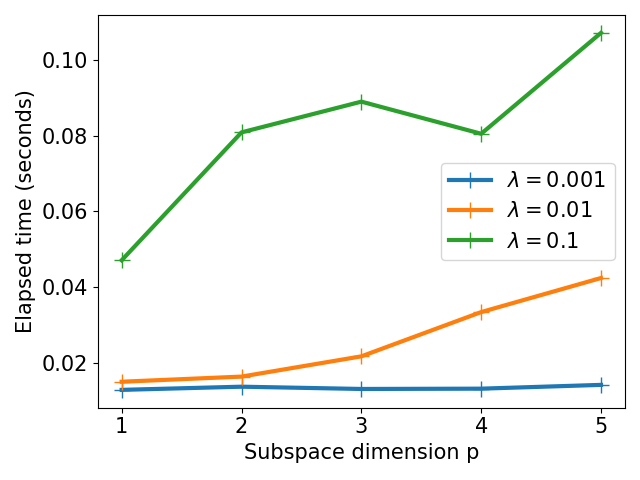}
    \includegraphics[width=0.32\linewidth]{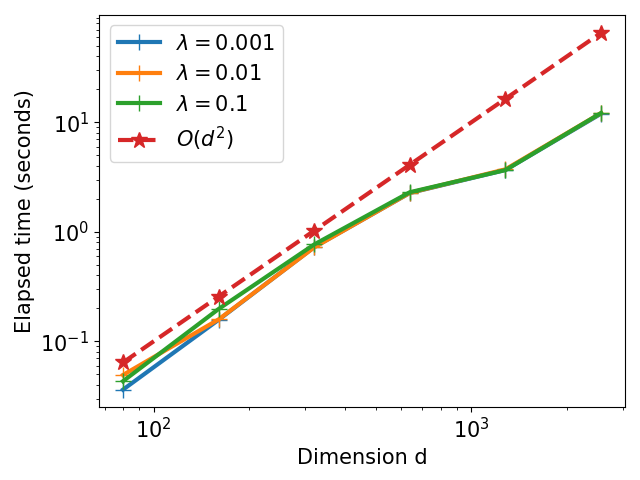}
    \includegraphics[width=0.32\linewidth]{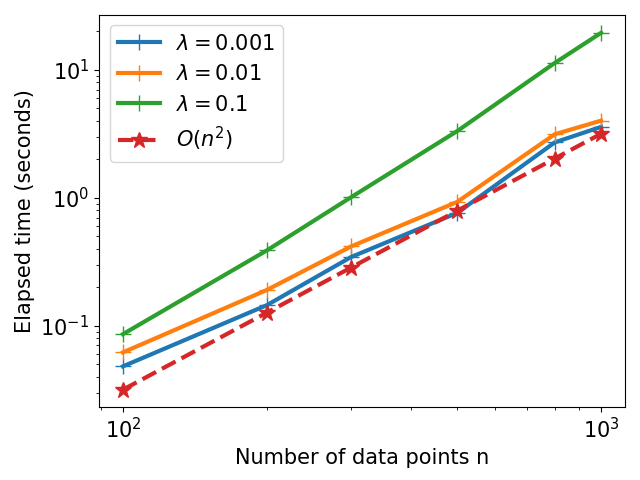}
    \caption{Scalability of WDA-nepv on the synthetic dataset: varying subspace dimension $p$ (left); varying dimension $d$ (middle); varying number of data points $n$ (right).}
    \label{fig:WDA_SYNTH_SCAL}
\end{figure}

Figure~\ref{fig:WDA_SYNTH_SCAL} illustrates that larger $\lambda$ incurs more running time.
As discussed previously, as $\lambda$ approaches zero, WDA-nepv becomes increasingly similar to LDA, which has a local quadratic convergence.
Therefore, convergence speed is generally faster for small values of $\lambda$. 
The left plot indicates that the running time of WDA-nepv scales linearly in subspace dimension $p$.
This linear scalability in $p$ is due to the matrix-vector multiplications involved in the exponential Euclidean distance matrices $\bbf{K}$ and the number of dominant eigenvectors in TRopts~\eqref{eq:TR} computation.
The scalability in data dimension $d$ and number of data points $n$ is displayed in the center and right log-log plots, respectively, in order to exemplify their quadratic scalability.
Moreover, a line of quadratic function is included in the log-log plots.
We observe that the slope of the lines corresponding to the running time of WDA-nepv match closely with the line of a quadratic function, indicating that the running time of WDA-nepv scales quadratically in data dimension $d$ and number of data points $n$. These quadratic scalability are due to the level-3 BLAS in the evaluation of the cross-covariance matrices $\bbf{C}_b(\bbf{P})$ and $\bbf{C}_w(\bbf{P})$~\eqref{eq:CbCwBlas3}.
}
\end{example}


\section{Conclusion and Future Works}\label{sec:conclusion}

\paragraph{Conclusion.}
We presented WDA-nepv, a new algorithm for Wasserstein discriminant analysis. WDA-nepv is a bi-level nonlinear eigenvector algorithm that fully utilizes the bi-level structure of WDA.
It employs an NEPv to compute the OT matrices at the inner optimization and another NEPv for the trace-ratio optimization at the outer iteration.
Both NEPvs can be solved efficiently by the SCF iteration. Unlike the existing algorithms WDA-gd and WDA-eig, WDA-nepv is derivative-free and surrogate-model-free. Specifically, in contrast to WDA-gd, WDA-nepv has no costs of computing the derivatives, which leads to lower running time. In contrast to WDA-eig, WDA-nepv solves the original problem directly, which results in higher accuracy. Additionally, we proposed an efficient level-3 BLAS implementation for computing cross-covariance matrices. We provided a convergence analysis of our algorithm, justifying the utilization of the SCF iteration for solving NEPvs. Numerical experiments demonstrate that WDA-nepv attains competitive classification accuracy and is scalable.

\paragraph{Future works.}
A main advantage of WDA is the ability to control global and local relations of the data points by varying the regularization parameter $\lambda$. However, there is no set technique to determine an optimal $\lambda_*$. Our numerical experiments indicate that $\lambda$ in the range from $0.01$ to $5$ generally works well for the classification task, but the optimal value varies across datasets. A systemic way of determining an optimal $\lambda$ is a subject of further study.

Our approach in solving NTRopt-WDA~\eqref{eq:WDA_tropt} was by fixating the dependence of the matrices on the projection and solving TRopts iteratively. Another approach that is worth further investigation is considering the first order necessary condition. Similar to the results of TRopt \cite{wang2007trace, zhang2010fast}, we can show that a local optimizer of NTRopt-WDA also satisfies a NEPv. This NEPv can then be solved by the SCF iteration. Our preliminary results of this approach on NTRopt-WDA indicate that it can outperform WDA-nepv in classification experiments. The major caveat of this approach, however, is its large computational costs in computing the derivatives. Improvement and further study of this approach will be left as our future work.

\bibliographystyle{plain}
\bibliography{refs} 

\end{document}